\documentclass[final, 12pt]{alt2021} % Anonymized submission
\usepackage{amsmath, amssymb, amsthm}
\usepackage{amsfonts}
\newtheorem{lemma}{Lemma}

\newtheorem{proposition}{Proposition}

\usepackage{float}

\usepackage{graphicx}

\usepackage[utf8]{inputenc}
\usepackage{url}            % simple URL typesetting
\usepackage[T1]{fontenc}    % use 8-bit T1 fonts
\usepackage{hyperref}       % hyperlinks
\usepackage{booktabs}       % professional-quality tables
\usepackage{nicefrac}       % compact symbols for 1/2, etc.
\usepackage{microtype}      % microtypography

\usepackage{caption}
\usepackage{subcaption}
\usepackage{enumitem}

\usepackage[]{xcolor}

\newcommand{\RR}{\mathbb{R}}
\newcommand{\NN}{\mathbb{N}}
\newcommand{\EE}{\mathbb{E}}
\newcommand{\Xcal}{\mathcal{X}}

\newcommand{\Dcal}{\mathcal{D}}
\newcommand{\Lcal}{\mathcal{L}}
\newcommand{\Ncal}{\mathcal{N}}
\newcommand{\Fcal}{\mathcal{F}}
\newcommand{\GP}{\textnormal{GP}}
\DeclareMathOperator*{\argmin}{argmin}

\renewcommand\cite{\citep}
\title[Uncertainty quantification using martingales for
misspecified Gaussian processes]{Uncertainty quantification using martingales\\
for misspecified Gaussian processes}
\usepackage{times}
\altauthor{%
 \Name{Willie Neiswanger} \Email{neiswanger@cs.stanford.edu}\\
 \addr Stanford University
 \AND
 \Name{Aaditya Ramdas} \Email{aramdas@stat.cmu.edu}\\
 \addr Carnegie Mellon University%
}

\begin{document}

\maketitle

\vspace{-2mm}
\begin{abstract}
We address uncertainty quantification for Gaussian processes (GPs) under misspecified priors, with an eye towards Bayesian Optimization (BO). GPs are widely used in BO because they easily enable exploration based on posterior uncertainty bands. However, this convenience comes at the cost of robustness: a typical function encountered in practice is unlikely to have been drawn from the data scientist's prior, in which case uncertainty estimates can be misleading, and the resulting exploration can be suboptimal.
% This brittle behavior is convincingly demonstrated in simple simulations.
We present a frequentist approach to GP/BO uncertainty quantification. We utilize the GP framework as a working model, but do not assume correctness of the prior. We instead construct a \emph{confidence sequence} (CS) for the unknown function using martingale techniques. There is a necessary cost to achieving robustness: if the prior was correct, posterior GP bands are narrower than our CS. Nevertheless, when the prior is wrong, our CS is statistically valid and empirically outperforms standard GP methods, in terms of both coverage and utility for BO. 
Additionally, we demonstrate that powered likelihoods provide robustness against model misspecification.
\end{abstract}

% \begin{keywords}%
%   List of keywords%
% \end{keywords}

\vspace{-0.1in}
\section{Introduction}
\label{sec:introduction}
\vspace{-0.05in}

In Bayesian optimization (BO), a Bayesian model is leveraged to optimize an unknown
function $f^*$ \cite{mockus1978application, shahriari2015taking, snoek2012practical}. 
One is allowed to query the function at various points $x$ in the domain, and get noisy observations of $f^*(x)$ in return. Most BO methods use a Gaussian process (GP) prior, with a chosen kernel
function. However, in practice, it may be difficult to specify the prior accurately.
% in advance, and thus the prior may be misspecified.
A few examples of where misspecification may arise include
\begin{itemize}[parsep=-1mm, topsep=2mm, leftmargin=8mm]
    \item an incorrect kernel choice (e.g. squared exponential versus Matern),
    \item bad estimates of kernel hyperparameters (e.g. lengthscale or signal variance), and
    \item heterogenous smoothness of $f^*$ over the domain $\mathcal{X}$.
    % (e.g. different lengthscales in different parts of $\Xcal$, due to heterogenous smoothness of $f^*$).
\end{itemize}
Each of these can yield misleading uncertainty estimates,
which may then negatively affect the performance of BO \cite{schulz2016quantifying, sollich2002gaussian}. \textbf{This paper instead presents a frequentist approach to uncertainty quantification for GPs} (and hence for BO), which uses martingale techniques to construct a confidence sequence (CS) for $f^*$, irrespective of
% the type or degree of
misspecification of the prior.
A CS is a sequence of (data-dependent) sets that are uniformly valid over time, meaning that $\{C_t\}_{t\geq 1}$ such that $\Pr(\exists t \in \NN: f^* \notin C_t )\leq \alpha$. The price of such a robust guarantee is that if the prior was indeed accurate, then our confidence sets are looser than those derived from the posterior.

\textbf{Outline} The next page provides a visual illustration of our contributions. Section \ref{sec:background} provides the necessary background on GPs and BO, as well as on martingales and confidence sequences. Section \ref{sec:mainresult} derives our prior-robust confidence sequence, as well as several technical details needed to implement it in practice. Section \ref{sec:simulations} describes the simulation setup used in Figure~\ref{fig:viz} in detail. We end by discussing related work and future directions in Section~\ref{sec:discussion}, with additional figures in the appendix.

% \clearpage

\begin{figure}[H]
\centering
\begin{subfigure}{0.45\textwidth}
  \includegraphics[width=\linewidth]{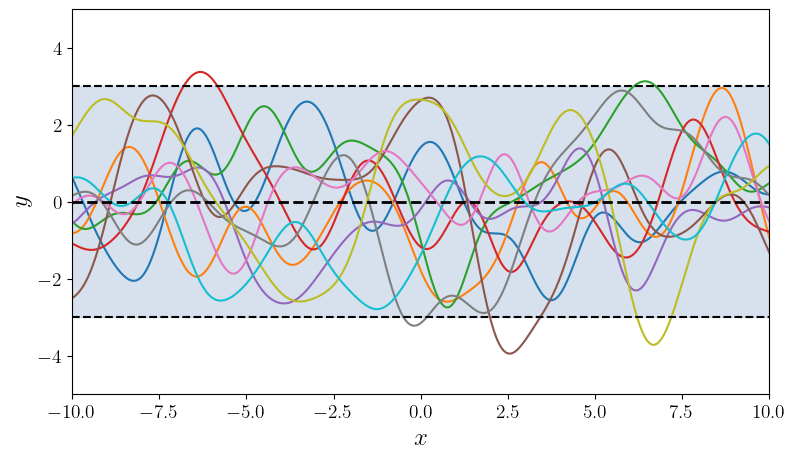}
%   \caption{True prior}
\end{subfigure}\hfil
\begin{subfigure}{0.45\textwidth}
  \includegraphics[width=\linewidth]{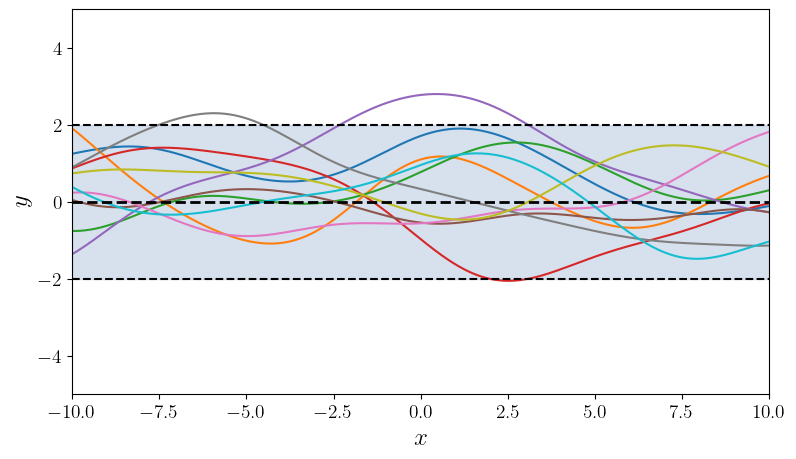}
%   \caption{Model prior}
\end{subfigure}
\\
\begin{subfigure}{0.45\textwidth}
  \includegraphics[width=\linewidth]{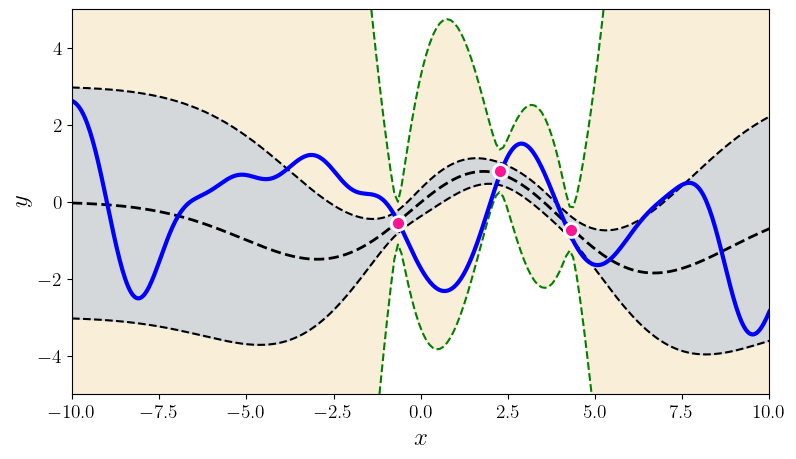}
%   \caption{$n=3$}
\end{subfigure}\hfil
\begin{subfigure}{0.45\textwidth}
  \includegraphics[width=\linewidth]{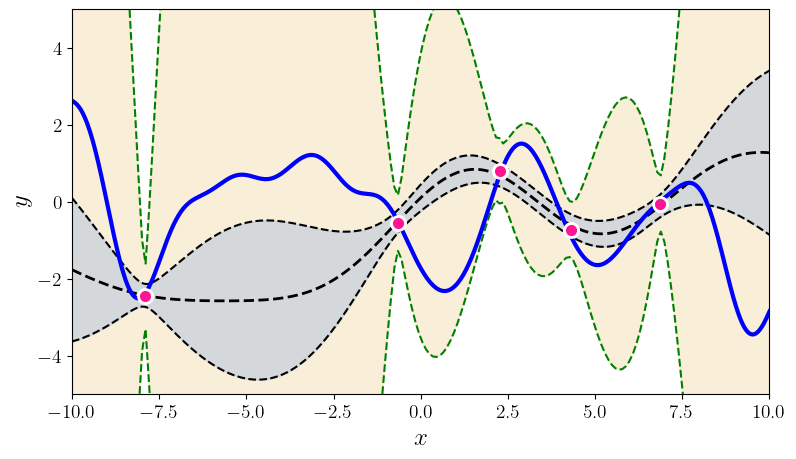}
%   \caption{$n=5$}
\end{subfigure}
\\
\begin{subfigure}{0.45\textwidth}
  \includegraphics[width=\linewidth]{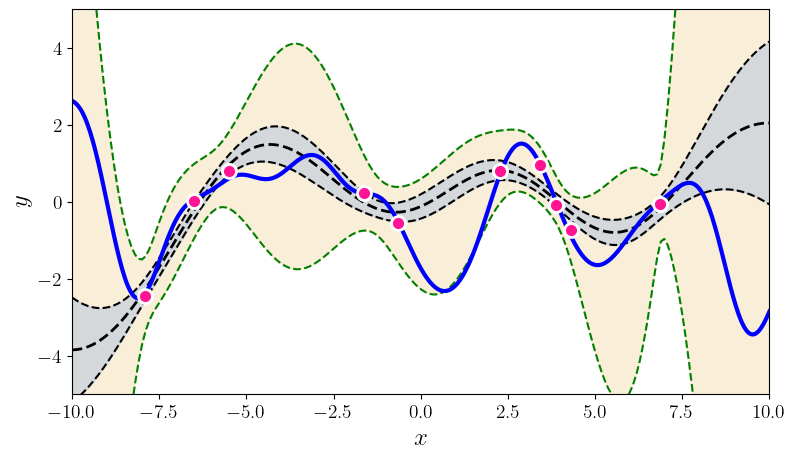}
%   \caption{$n=15$}
\end{subfigure}\hfil
% \medskip
\begin{subfigure}{0.45\textwidth}
  \includegraphics[width=\linewidth]{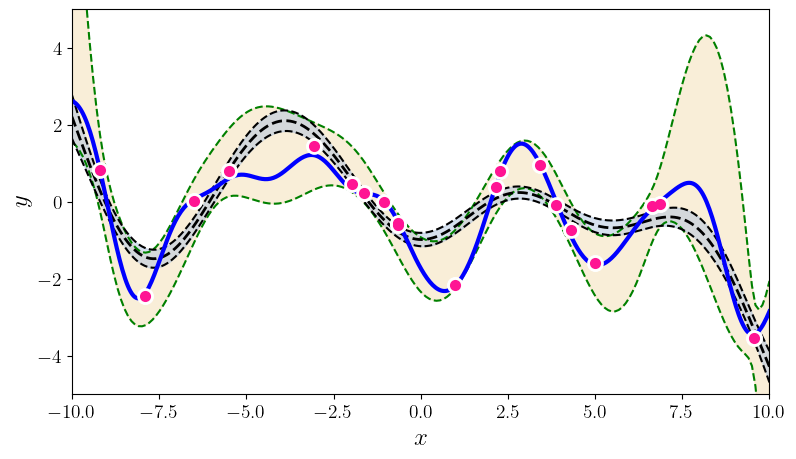}
%   \caption{$n=17$}
\end{subfigure}
\\
\begin{subfigure}{0.45\textwidth}
  \includegraphics[width=\linewidth]{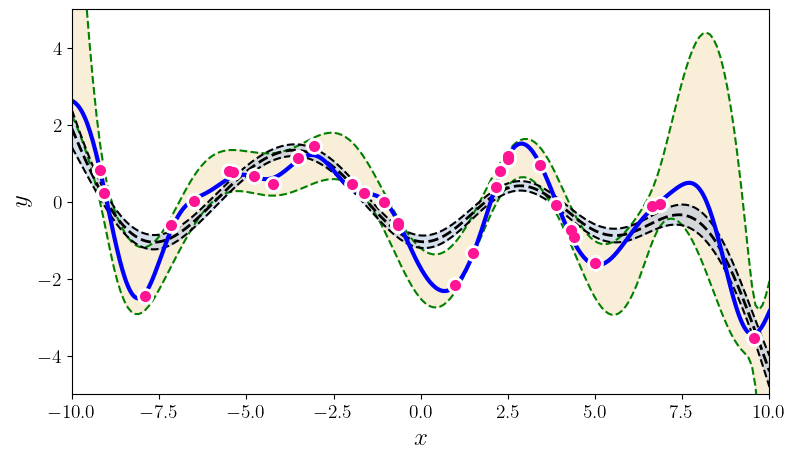}
%   \caption{$n=25$}
\end{subfigure}\hfil
\begin{subfigure}{0.45\textwidth}
  \includegraphics[width=\linewidth]{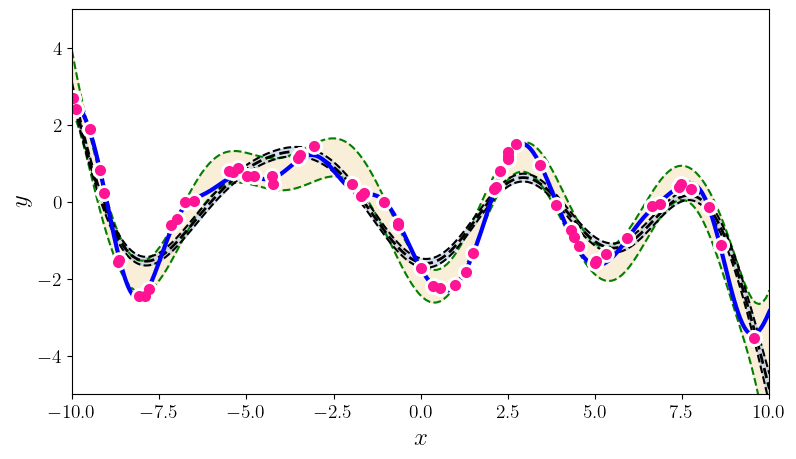}
%   \caption{$n=40$}
\end{subfigure}
\vspace{-4mm}
\caption{ 
\small
This figure summarizes the paper's contributions.
The top two plots show various random functions drawn from a GP prior with hyperparameter settings A (left) and B (right). Then, a single function (blue curve) is drawn using prior A, and is fixed through the experiment. The pink dots are the observations; there are 3, 5, 10, 20, 30, 60 pink dots in the bottom 6 plots. The grey shaded region shows the standard GP posterior when (mistakenly) working with prior B. The brown shaded region shows our new confidence sequence, also constructed with the wrong prior B. The brown region is guaranteed to contain the true function with high probability uniformly over time. 
The grey confidence band after just 3 observations is already (over)confident, but quite inaccurate, and it never recovers. The brown confidence sequence is very wide early on (perhaps as it should be) but it recovers as more points are drawn. 
Thus, the statistical price of robustness to prior misspecification is wider bands. Whether this is an acceptable tradeoff to the practitioner is a matter of their judgment and confidence in the prior.
The observation that the posterior never converges to the truth (the data does not wash away the wrong prior) appears to be a general phenomenon of failure of the Bernstein-von-Mises theorem in infinite dimensions~\cite{freedman1999wald,cox1993analysis}.
The rest of this paper explains how this confidence sequence is constructed, using the theory of martingales. We provide more details (such as kernel hyperparameters A and B) for this simulation in Section \ref{sec:simulations}. Simulations for BO are available in the supplement. 
}
\label{fig:viz}

\end{figure}

% \newpage

\section{Mathematical background}
\label{sec:background}

\paragraph{Gaussian Processes (GP).} A GP is a stochastic process (a collection of random variables indexed by domain $\Xcal$) such that every finite collection of those random variables has a multivariate normal distribution. The distribution of a GP is a distribution over functions $g: \Xcal \mapsto \RR$, and thus GPs are often used as Bayesian priors over unknown functions. A GP is itself typically
specified by a mean function $\mu:\Xcal\rightarrow\RR$ and a
covariance kernel $\kappa:\Xcal^2\rightarrow\RR$. 
Suppose we draw a function 
\begin{equation}\label{eq:prior}
f \sim \text{GP}(\mu,\kappa)
\end{equation}
and obtain a set of $n$
observations $D_n = \{(X_i,Y_i)\}_{i=1}^n$, where $X_i \in \Xcal$,
\begin{equation}\label{eq:model}
Y_i = f(X_i) + \epsilon_i \in \RR, \text{ and } \epsilon_i \sim \Ncal(0, \eta^2).
\end{equation}
Then, the posterior process $f|D_n$ is also a GP with mean function
$\mu_n$ and covariance kernel
$\kappa_n$, described as follows. Collect the $Y_i$s into a vector $Y\in\RR^n$, and define
$k,k'\in\RR^n$ with $k_i =
\kappa(x,X_i), k'_i=\kappa(x',X_i)$,
and $K\in \RR^{n\times n}$ with
$K_{i,j} = \kappa(X_i,X_j)$. We can then write $\mu_n, \kappa_n$ as
\begin{align}
\hspace{-0.05in}
\mu_n(x) &= k^\top(K + \eta^2I)^{-1}Y, \hspace{0.35in}
\kappa_n(x,x') = \kappa(x,x') - k^\top(K + \eta^2I)^{-1}k'.
\hspace{0.05in}
\label{eqn:gppost}
\end{align}
Further background on GPs can be found in~\citet{williams2006gaussian}. In this paper, we describe a simple method for inference when \eqref{eq:prior} does not hold, but \eqref{eq:model} holds; in other words, the prior is arbitrarily misspecified but the model is correct. (If both are correct, GPs work fine, and if both are arbitrarily misspecified, statistical inference is essentially impossible.)

\paragraph{Bayesian Optimization (BO).}
Suppose we wish to minimize an unknown, fixed, nonrandom, function $f^*$ over
a domain $\Xcal$. Bayesian optimization (BO) leverages probabilistic models to 
perform optimization by assuming that $f^*$ was sampled from a GP. 

At time $t$ (we switch from $n$ to $t$ to emphasize temporality), assume we have already evaluated $f^*$ at points
$\{X_i\}_{i=1}^{t-1}$ and obtained observations $\{Y_i\}_{i=1}^{t-1}$.
To determine the next domain point $X_t$ to evaluate,
we first use the posterior GP to define an \emph{acquisition function}
$\varphi_t:\Xcal\rightarrow\RR$,
which specifies the utility of evaluating $f^*$ at any $x\in\Xcal$.
We then minimize the acquisition function to yield $X_t=\argmin_{x \in \Xcal} \varphi_t(x)$,
and evaluate $f^*$ at $X_t$. One of the most commonly used acquisition functions
is the GP lower confidence bound\footnote{Often described as the GP upper confidence bound (GP-UCB),
we use the GP lower confidence bound (GP-LCB) since we are performing minimization.}
(GP-LCB) by~\citet{srinivas2009gaussian}, written
\begin{align}
\varphi_t(x) = \mu_t(x) - \beta_t^{1/2} \sigma_t(x)
\label{eqn:ucbacq}
\end{align}
where $\mu_t$ and $\sigma_t$ are the posterior GP mean and
standard deviation, and $\beta_t>0$ is a tuning parameter that 
determines the tradeoff between exploration and exploitation.

% is the expected improvement (EI)~\citep{movckus1975bayesian}, written
% \begin{align}
% \varphi_t(x) = \EE\left[\max\{0, f^*(x) - m_{t-1}\} \big|\{(X_i, Y_i)\}_{i=1}^{t-1}
%             \right],
% \label{eqn:eiacq}
% \end{align}
% which measures the expected improvement over the current maximum value
% according to the posterior GP. Here, $m_{t-1} = \argmax_{i\leq t-1} f^*(x_i)$ denotes
% the current best value.

Due to inheriting their worldview from GPs, theoretical guarantees in the BO literature typically assume correctness of both \eqref{eq:prior} and \eqref{eq:model}. These may or may not be reasonable assumptions. In this paper, \eqref{eq:prior} is used as a working model that is not assumed correct, but \eqref{eq:model} is still assumed. We do not provide guarantees on any particular BO algorithm persay, but instead provide correct uncertainty quantification that could be exploited by any BO algorithm, including but not necessarily GP-LCB.

\paragraph{Filtrations and stopping times.} To make the sequential aspect of BO explicit, let 
\[
\Dcal_t = \sigma((X_1,Y_1),\dots,(X_t,Y_t)) \equiv \sigma(D_t)
\] 
denote the sigma-field of the first $t$ observations, which captures the information known at time $t$; $\Dcal_0$ is the trivial sigma-field. Since $\Dcal_t \supset \Dcal_{t-1}$,  $\{\Dcal_t\}_{t \geq 0}$ forms a filtration (an increasing sequence of sigma-fields). Using this language, the acquisition function $\varphi_t$ is then \emph{predictable}, written $\varphi_t \in \Dcal_{t-1}$, meaning that it is measurable with respect to $\Dcal_{t-1}$ and is hence determined with only the data available after $t-1$ steps. As a result, $X_t$ is technically also predictable. However, $Y_t$ is not predictable (it is \emph{adapted}), since $Y_t \in \Dcal_t$ but $Y_t \notin \Dcal_{t-1}$.  A stopping time $\tau$ is  an $\NN$-valued random variable such that $\{\tau \geq t \} \in \Dcal_{t-1}$ for all $t$, or equivalently if 
\[\{\tau \leq t\} \in \Dcal_t,\] 
meaning that we can tell if we have stopped by time $t$, using only the information available up to $t$. 

\paragraph{Martingales.} An integrable process $\{M_t\}_{t \geq 0}$ is said to be a martingale with respect to filtration $\{\Dcal_t\}_{t \geq 0}$, if $M_t \in \Dcal_t$ and for every $t \geq 1$, we have
\[
\EE[M_t | \Dcal_{t-1}] = M_{t-1}.
\] 
% As a result, for fixed times $t$, we have $\EE[M_t] = M_0$. The optional stopping theorem tells us that for (random, data-dependent) stopping times $\tau$ (with respect to the same filtration) that are almost surely bounded, we have $\EE{M_\tau} = M_0$. 
If we replaced the $=$ above by an inequality $\leq$, the resulting process is called a supermartingale. Every martingale is a supermartingale but not vice versa. Ville's inequality~\cite[Pg. 100]{ville_etude_1939} states that if $\{M_t\}$ is a nonnegative supermartingale, then for any $x > 0$, we have
\begin{equation}\label{eq:Ville}
\small
\Pr (\exists t \in\NN: M_t \geq x) \leq \frac{\EE[M_0]}{x}.
\end{equation}
See \citet[Lemma 1]{howard_exponential_2018} for a measure-theoretic proof and \citet[Proposition 14.8]{shafer2019game} for a game-theoretic variant.
In many statistical applications,  $M_0$ is chosen to deterministically equal one.
Ville's may be viewed as a time-uniform version of Markov's inequality for nonnegative random variables. In the next section, we will construct a martingale (hence supermartingale) for GP/BO, and construct a confidence sequence (defined next) for the underlying function $f$ by applying Ville's inequality.

\paragraph{Confidence sequences (CS).} 
Suppose we wish to sequentially estimate an unknown quantity $\theta^*$ (a scalar, vector, function, etc.)  as we observe an increasing number of datapoints, summarized as a filtration $\Dcal_t$.
A CS is defined as a sequence of confidence sets $\{C_t\}_{t \geq 1}$ that contains $\theta^*$ at all times with high probability. 
Formally, for a confidence level $\alpha \in (0,1)$, we need that $C_t \in \Dcal_t$ and
\begin{equation}\label{eq:CS}
\underbrace{\Pr(\forall t \in \mathbb{N}: \theta^* \in C_t)}_{\text{coverage at all times}} \geq 1-\alpha \quad \equiv \quad 
\underbrace{\Pr(\exists t \in \mathbb{N}: \theta^* \notin C_t)}_{\text{error at some time}} \leq \alpha.
\end{equation}
Here, $C_t$ obviously depends on $\alpha$, but it is suppressed for simplicity. Importantly, property \eqref{eq:CS} holds if and only if $\Pr(\theta^* \in C_\tau) \geq 1-\alpha$ for all possible (potentially infinite) stopping times $\tau$. This allows us to provide correct uncertainty quantification that holds even at data-dependent stopping times. Next, we describe  our construction of a confidence sequence for $f^*$.

\section{Deriving our prior-robust confidence sequence}
\label{sec:mainresult}

One of the roles of the prior in BO is to restrict the complexity of the function $f^*$.  Since we do not assume the prior is well-specified, we need some other way to control the complexity of $f^*$---without any restriction on $f^*$, we cannot infer its value at any point outside of the observed points since it could be arbitrarily different even at nearby points. We do this by assuming that $f^* \in \Fcal$ for some $\Fcal$, that is either explicitly specified---say via a bound on the Reproducing Kernel Hilbert Space (RKHS) norm, or by a bound on the Lipschitz constant---or implicitly specified (via some kind of regularization). 

The choice of $\Fcal$ has both statistical and computational implications; the former relates to the size of the class, the sample complexity of identifying the optimum via BO, and the rate at which the confidence bands will shrink, while the latter relates to how much time it takes to calculate and/or update the confidence bands. Ultimately, $\Fcal$ must be specified by the practitioner based on their knowledge of the underlying problem domain. For this section, we treat any arbitrary $\Fcal$, and in the next section we discuss one particular choice of $\Fcal$ for which the computational load is reasonable.

It is worth noting that we have not simply shifted the problem from specifying a prior to specifying $\Fcal$---the latter does not impose a probability structure amongst its elements, while the former does. There are other differences as well; for example comparing a GP prior with a particular kernel, to a bounded RKHS ball for the same kernel, we find that the former is much richer than the latter---as mentioned after Theorem~3 of~\citet{srinivas2009gaussian}, random functions drawn from a GP have infinite RKHS norm almost surely, making the sample paths much rougher/coarser than functions with bounded RKHS norm.

\subsection{Constructing the prior-posterior-ratio martingale}
\label{sec:constructing}

We first begin with some technicalities. 
% Let $\Xcal \subset \RR^d$ be a compact domain for simplicity. 
Recall that a GP is interpreted as a prior distribution over functions $g: \Xcal \mapsto \RR$.
For simplicity, and to avoid measure-theoretic issues, consider the case of $\Xcal = \RR^d$ by default, equipped with the Borel sigma algebra. It is clear to us that the following results do hold more generally, albeit at the price of further mathematical machinery, since extra care is needed when dealing with infinite-dimensional measures.
Let 
$\GP_0(f)$ represent the prior ``density'' at function $f$, and let $\GP_t(f)$ represent the posterior ``density'' at $f$ after observing $t$ datapoints. ``Density'' is in quotes because in infinite dimensional spaces, there is no analog of the Lebesgue measure, and thus it is a~priori unclear which measure these are densities with respect to. Proceeding for now, we soon sort this issue out.

Define the \emph{prior-posterior-ratio} for any function $f$ as the following real-valued process:
\begin{equation}\label{eq:PPR}
R_t(f) := \frac{\GP_0(f)}{\GP_t(f)}.
\end{equation}
Note that $R_0(f)=1$ for all $f$.
Denote the working likelihood of $f$ by
\begin{equation}\label{eq:likelihood}
\Lcal_{t}(f) ~:=~ \prod_{i=1}^t \frac1{\eta\sqrt{2\pi}} e^{-\frac12\left(\frac{Y_i - f(X_i)}{\eta}\right)^2} ~\equiv~ \prod_{i=1}^t \frac1{\eta} \phi\left(\frac{Y_i - f(X_i)}{\eta}\right),
\end{equation}
where $\phi(y)$ denotes the standard Gaussian PDF, so that $\phi((y - \mu)/\sigma)/\sigma$ is the PDF of $\Ncal(\mu,\sigma^2)$.
Then, for any function $f$, the working posterior GP is given by 
\begin{equation}\label{eq:posterior}
\GP_{t}(f) := \frac{\GP_0(f) \Lcal_t(f)}{\int_g \GP_0(g) \Lcal_{t}(g)}.
\end{equation}
Substituting the posterior \eqref{eq:posterior} and likelihood \eqref{eq:likelihood} into the definition of the prior-posterior-ratio \eqref{eq:PPR}, the latter can be more explicitly written as \begin{equation}\label{eq:PPR-mixtureLRT}
R_t(f) ~=~ \int_g \GP_0(g) \frac{\Lcal_{t}(g)}{\Lcal_{t}(f)} ~\equiv~ \EE_{g \sim \GP_0}\left[\frac{\Lcal_{t}(g)}{\Lcal_{t}(f)} \right],
\end{equation}
and it is this last form that we use, since it avoids measure-theoretic issues. 
% Indeed, the likelihood $\Lcal_t(f)$ is well-defined for any function $f$.
Indeed, $\Lcal_t(f), R_t(f)$ are well-defined and finite for every $f$, as long as $f$ itself is finite, and one is anyway uninterested in considering functions that can be infinite on the domain.
% It is possible to get around this issue by noting that the prior is absolutely continuous with respect to the posterior, allowing us to define Radon-Nikodym derivatives, or by discretizing the domain and dealing with finite dimensional Gaussians, but we will see that neither of these will be necessary, and the intuition of densities is only needed for metaphorical reasons. 

As mentioned at the start of this section, fix a function $f^* \in \Fcal$. Assume that the data are observed according to \eqref{eq:model} when the $X_i$s are predictably chosen according to any acquisition function. Despite not assuming \eqref{eq:prior}, we will still use a GP framework to model and work with this data, and we call this our ``working prior'' to differentiate it from an assumed prior.

\begin{lemma}\label{lem:PPR-martingale}
Fix any arbitrary $f^* \in \Fcal$, and assume data-generating model \eqref{eq:model}. Choose any acquisition function $\varphi_t$, any working prior $\GP_0$ and construct the working posterior $\GP_{t}$. Then, the prior-posterior-ratio at $f^*$, denoted $\{R_t(f^*)\}_{t \geq 0}$, is a martingale with respect to filtration $\{\Dcal_t\}_{t \geq 0}$. 
% meaning that
% \[
% \EE_{\Dcal_t \sim f^*}\left[\frac{\GP_0(f^*)}{\GP_{t}(f^*)} \mid \Dcal_{t-1}\right] ~ = ~ \frac{\GP_0(f^*)}{\GP_{t-1}(f^*)},
% \]
% where $\EE_{\Dcal_t \sim f^*}$ denotes expectation with respect to the random observations in \eqref{eq:model}.
\end{lemma}
% To avoid measure-theoretic issues involving the Radon-Nikodym derivative as mentioned earlier, we provide the proof when $f \in \RR^G$, so evaluating $f(X_i)$ involves reading off the value of $f$ at the index corresponding to gridpoint $X_i$. To prepare for the proof, let us set up some notation. 

\begin{proof}
Evaluating $R_t$ at $f^*$, taking conditional expectations and applying Fubini's theorem, yields
\begin{align*}
\EE_{D_t \sim f^*}[R_t(f^*) \mid \Dcal_{t-1}] 
&=~ \EE_{D_t \sim f^*}\left[\EE_{g \sim \GP_0} \left[\frac{\Lcal_{t}(g)}{\Lcal_{t}(f^*)} \right]\mid \Dcal_{t-1}\right]\\ 
&=~  \EE_{g \sim \GP_0} \left[ \EE_{D_t \sim f^*}\left[ \frac{\Lcal_{t}(g)}{\Lcal_{t}(f^*)}\mid \Dcal_{t-1}\right]\right]\\
&\stackrel{(i)}{=}~ \EE_{g \sim \GP_0} \left[ \frac{\Lcal_{t-1}(g)}{\Lcal_{t-1}(f^*)} ~\cdot~ 
\underbrace{\EE_{D_t \sim f^*}\left[ \frac{\phi(\frac{Y_t - g(X_t)}{\eta})}{\phi(\frac{Y_t - f^*(X_t)}{\eta})}\mid \Dcal_{t-1}\right]}_{=1}
\right] ~=~ R_{t-1}(f^*),
\end{align*}
where equality $(i)$ follows because $X_t \in \Dcal_{t-1}$ by virtue of the acquisition funtion being predictable. To conclude the proof, we just need to argue that the braced term in the last expression equals one as claimed. This term can be recognized as integrating a likelihood ratio, which equals one because for any two absolutely continuous distributions $P,Q$, we have $\EE_P(dQ/dP) = \int (dQ/dP) dP = \int dQ = 1$. For readers unfamiliar with this fact, we verify it below by direct integration. Once we condition on $\Dcal_{t-1}$, only $Y_t$ is random, and so the relevant term equals
\[
% \EE_{f^*}\left[ \frac{\phi(\frac{Y_t - g(X_t)}{\eta})}{\phi(\frac{Y_t - f^*(X_t)}{\eta})}\mid \Dcal_{t-1}\right] = 
\int_y \frac{\phi(\frac{y - g(X_t)}{\eta})}{\phi(\frac{y - f^*(X_t)}{\eta})} \frac1{\eta}\phi\left(\frac{y - f^*(X_t)}{\eta}\right) dy = \int_y \frac1{\eta} \phi\left(\frac{y - g(X_t)}{\eta}\right)  dy = 1,
\]
where the last equality holds simply because a Gaussian PDF with any mean integrates to one.
\end{proof}

Also see \citet{waudby2020confidence} for another application of the prior-posterior ratio martingale.
The prior-posterior-ratio is related to the marginal likelihood and the Bayes factor, but the latter two terms are typically used in a Bayesian context, so we avoid their use since the guarantee above is fully frequentist: the expectation $\EE_{D_t \sim f^*}$ is not averaging over any prior: no prior is even assumed to necessarily exist in generating $f^*$, or if it exists it may be incorrectly specified. The most accurate analogy to past work in frequentist statistics is to interpret this statement as saying that the mixture likelihood ratio is a martingale --- a well known fact, implicit in  \citet{wald_sequential_1947}, and exploited in sequential testing \citep{robbins_boundary_1970} and estimation \citep{howard_uniform_2019}. Here, the prior $\GP_0$ plays the role of the mixing distribution. 
However, our language more directly speaks to how one might apply Bayesian methodology towards frequentist goals in other problems.

\subsection{Constructing the confidence sequence}

Despite the apparent generality of Lemma~\ref{lem:PPR-martingale}, it is not directly useful. Indeed, $R_t(f^*)$ is a martingale, but not $R_t(f)$ for any other $f$, and we obviously do not know $f^*$. This is where Ville's inequality \eqref{eq:Ville} enters the picture: we use Lemma~\ref{lem:PPR-martingale} to construct the following confidence sequence and use Ville's inequality to justify its correctness. Define 
\begin{equation}\label{eq:CS-PPR}
C_t := \left\{ f \in \Fcal : R_t(f) < \frac1\alpha \right\}.
\end{equation}
We claim that $f^*$ is an element of the confidence set $C_t$, through all of time, with high probability.

\begin{proposition}
Consider any (fixed, unknown) $f^* \in \Fcal$ that generates data according to \eqref{eq:model}, any acquisition function $\varphi_t$, and any nontrivial working prior $\GP_0$. Then, $C_t$ defined in \eqref{eq:CS-PPR} is a confidence sequence for $f^*$:
\[
\Pr(\exists t \in \NN: f^* \notin C_t) \leq \alpha.
\]
Thus, at any arbitrary data-dependent stopping time $\tau$, we have $\Pr(f^* \notin C_\tau) \leq \alpha$.
\end{proposition}
\begin{proof}
First note that $f^* \notin C_t$ if and only if $R_t(f^*) \geq 1/\alpha$. Recall that $R_t(f^*)$ is a nonnegative martingale by Lemma~\ref{lem:PPR-martingale}, and note that $R_0(f^*)=1$. Then, Ville's inequality \eqref{eq:Ville} with $x=1/\alpha$ implies that $\Pr(\exists t \in \NN: R_t(f^*) \geq 1/\alpha) \leq \alpha$.
\end{proof}

$C_t$ is our prior-robust confidence sequence for $f^*$. For the purposes of the following discussion, let $|C_t|$ denote its size, for an appropriate notion of size such as an $\epsilon$-net covering.
Intuitively, if the working prior $\GP_0$ was accurate, which in the frequentist sense means that it put a large amount of mass at $f^*$ relative to other functions, then $|C_t|$ will be (relatively) small. If the working prior $\GP_0$ was inaccurate, which could happen because of a poor choice of kernel hyperparameters, or a poor choice of kernel itself, then $|C_t|$ will be (relatively) large. This degradation of quality ($|C_t|$ relative to accuracy of the prior) is smooth, in the sense that as long as small changes in the GP hyperparameters only change the mass at $f^*$ a little bit, then the corresponding confidence sequence (and hence its size) will also change only slightly. Formalizing these claims is possible by associating a metric over hyperparameters, and proving that if the map from hyperparameters to prior mass is Lipschitz, then the map $|C_t|$ is also Lipschitz, but this is beyond the scope of the current work. Such ``sensitivity analysis'' can be undertaken if the proposed new ideas are found to be of interest. 

$C_t$ is a confidence band for the entire function $f^*$, meaning that it is uniform over both $\Xcal$ and time, meaning that it provides a confidence interval for $f^*(x)$ that is valid simultaeously for all times and for all $x$ (on the grid, for simplicity). This uniform guarantee is important in practice because the BO algorithm is free to query at any point, and also free to stop at any data-dependent stopping time.

% We have not seen martingales being used in the above fashion (to quantify uncertainty) in the GP literature. 
The aforementoned proposition should be compared to~\citet[Theorem 6, Appendix B]{srinivas2009gaussian}, which is effectively a confidence sequence for $f$ (though they did not use that terminology), and yielded the regret bound in their Theorem~3, which is very much in the spirit of our paper. However, the constants in their  Theorems~3, 6 are very loose, and it is our understanding that these are never implemented as such in practice; in contrast, our confidence sequence is essentially tight, with the error probability almost equaling $\alpha$, because Ville's inequality almost holds with equality for nonnegative martingales (it would be exact equality in continuous time).

Martingales have also been used in other fashions, for example to analyze convergence properties of BO methods; for example, \citet{bect2019supermartingale} use (super)martingales to study consistency of sequential uncertainty reduction strategies in the well-specified case.

% \vspace{-0.05in}
\section{Practical considerations and numerical simulations}
\label{sec:simulations}
% \vspace{-0.1in}

Being an infinite dimensional confidence set containing uncountably many functions, even at a fixed time, $C_t$ cannot be explicitly stored on a computer. In order to actually use $C_t$ in practice, two critical questions remains: (a) returning to the very start of Section~\ref{sec:mainresult}, how should we pick the set of functions $\Fcal$ under consideration? (b) at a fixed time $t$, and for a fixed new test point $x$ under consideration by the acquisition function for a future query, how can we efficiently construct the confidence interval for $f^*(x)$ that is induced by $C_t$? 
These two questions are closely tied together: certain choices of $\Fcal$ in (a) may make step (b) harder.
There cannot exist a single theoretically justified way of answering question (a): the type of functions that are ``reasonable'' will depend on the application.

We describe our approach to tackling these questions in the context of Figure~\ref{fig:viz}. Our answer ties together (a) and (b) using a form of implicit regularization; we suspect there is room for improvement.
Our code is available at:  \url{https://github.com/willieneis/gp-martingales}

\subsection{The introductory simulation}
% \vspace{-0.1in}
In Figure~\ref{fig:viz}, we define two gaussian processes priors, $\text{GP}_0^{(1)}(\mu_1, \kappa_1)$
and $\text{GP}_0^{(2)}(\mu_2, \kappa_2)$. Both covariance matrices $\kappa_1$ and $\kappa_2$
are defined by a squared exponential kernel, i.e.
\begin{align}
\kappa(x, x’) = \sigma^2 \text{exp} \left(- \frac{(x-x’)^2}{2 \ell^2} \right),
\end{align}
with lengthscale $\ell$ and signal variance $\sigma^2$.
In this example, $\kappa_1$ has parameters $\{\ell = 1, \sigma^2 = 1.5\}$ and 
$\kappa_2$ has parameters $\{ \ell = 3, \sigma^2 = 1 \}$. Both
GPs have a fixed noise variance $\eta^2 = 0.1$ in model \eqref{eq:model}.
We show the posterior $95\%$ confidence region and posterior
samples for $\text{GP}_0^{(1)}(\mu_1, \kappa_1)$ and
for $\text{GP}_0^{(2)}(\mu_2, \kappa_2)$ in Figure~\ref{fig:viz};
the top two plots show typical functions drawn from these priors.

Now, we draw a single function from the first prior,
$f^* \sim \text{GP}_0^{(1)}(\mu_1, \kappa_1)$
shown as a blue line, which we really treat as a fixed function in this paper.
We then draw $t$ observations from this function via
\begin{align}
    X_i \sim \text{Uniform}\left[-10, 10\right], \hspace{5mm}
    Y_i \sim \mathcal{N}\left( f^*(X_i), \eta^2 \right), \hspace{5mm}
    i=1,\ldots,t. \nonumber
\end{align}
We compute the posterior $\GP_t$ (Eq.~\ref{eqn:gppost}), under 
the second prior $\text{GP}_0^{(2)}(\mu_2, \kappa_2)$, and plot the 
$95\%$ confidence region for $t \in (3, 5, 15, 17, 25, 40)$ 
in Figure~\ref{fig:viz}, rows 2-4 (shown as blue shaded regions).

We then aim to construct the prior-robust confidence sequence.
For each $t$, we can write the prior-posterior-ratio and confidence sequence
for $\alpha=0.05$ as 
\begin{align}
  R_t(f) = \frac{ \text{GP}_0^{(2)}(f) }{ \text{GP}_t(f) },
  \hspace{2mm} \text{and} \hspace{3mm}
  C_t = \left\{ f \in \Fcal : R_t(f) < 20 \right\}.
\end{align}
Next, we describe our procedure for implicitly specifying $\Fcal$ while computing $C_t$ in
Section~\ref{sec:numericaltools}, and plot it for each $x \in [-10, 10]$ in
Figure~\ref{fig:viz} (shown as yellow/brown shaded regions).

\subsection{Implicit specification of $\mathcal{F}$ while computing the confidence interval for $f^*(x)$ at time $t$}
\label{sec:numericaltools}
% \vspace{-0.05in}

Suppose we are at iteration $t$ of BO, using a Bayesian model with prior
$\text{GP}_0(\mu_0, \kappa_0)$. Assume that we have observed data
${D}_{t-1} = \{(X_i, Y_i)\}_{i=1}^{t-1}$.
Assume we have a sequence $X_1', X_2^\prime, ... \in \mathcal{X}$ 
over which we'd like to evaluate our acquisition function $\varphi_t(x)$.
% In our procedure to compute $C_t$, we move sequentially through domain $\mathcal{X}$
% according to some sequence $X_1^\prime, X_2^\prime, ... \in \mathcal{X}$.
In BO, this sequence would typically be determined by an acquisition optimization routine,
which we can view as some zeroth order optimization algorithm.
For each point $X'$ in this sequence we do the following.

\paragraph{(1) Compute the GP posterior.}
Let $G_t = \{X \in {D}_{t-1}\} \cup X^\prime$. We will restrict the prior and posterior 
GP to this set of grid points, making them finite but high-dimensional Gaussians.
The infinite-dimensional confidence sequence (or a confidence set at one time instant) for $f^*$ induces a finite-dimensional confidence sequence (set) for its function values at these gridpoints.
In other words, for computation tractability, instead of computing the confidence set for the whole function,
we can think of each function as $f \in \mathbb{R}^{|G_t|}$, and compute posterior $\text{GP}_t(\mu_t, \kappa_t)$ according to
Eq.~\ref{eqn:gppost}. To avoid unnecessary notation, we will still call the gridded function as $f$ and its induced confidence set as $C_t$ (though in this section they will be $G_t$-dimensional).

\paragraph{(2) Regularize the posterior-prior ratio.}
% \paragraph{(2) Define the confidence sequence $C_t$.}
We first define $\widetilde{\text{GP}}_0(\widetilde{\mu}_0, \widetilde{\kappa}_0)$
to be a GP that is very similar to the prior, except slidely wider.
More formally, let $\widetilde{\text{GP}}_0 \geq \text{GP}_0$ according to
Loewner order, so that $\widetilde{K}_0 - K_0$ is positive semi-definite (where
$\widetilde{K}_0$ and $K_0$ are the covariances matrices associated with
$\widetilde{\kappa}_0$ and $\kappa_0$). In our experiment, we let $\widetilde{\kappa}_0$
have the same parameters as $\kappa_0$, except with a slight larger signal
variance (e.g. $(1 + \gamma)\sigma^2$, where $\gamma = 10^{-2}$).

One can prove that there exists a Gaussian distribution with density
proportional to $\text{GP}_t(f) / \widetilde{\text{GP}}_0(f)$.
Define $\widetilde{R}_t^{-1}(f) := \text{GP}_t(f) / \widetilde{\text{GP}}_0(f)
= c \mathcal{N}(f | \mu_c, \Sigma_c)$, where $c > 0$. Then
\begin{align*}
  \Sigma_c = \left( K_t^{-1} - \widetilde{K}_0^{-1} \right)^{-1},
  \hspace{2mm} %\text{and}\hspace{2mm}
  \mu_c =  \Sigma_c \left( K_t^{-1} \mu_t -
                    \widetilde{K}_0^{-1} \widetilde{\mu}_0 \right),
  \hspace{2mm}\text{and}\hspace{2mm}
    c = \frac{|\widetilde{K}_0|  \mathcal{N}(\mu_t | \mu_0, \widetilde{K}_0 - K_t)}
            {|\widetilde{K}_0 - K_t|}
\end{align*}
where $K_t$ and $\widetilde{K}_0$ are the covariance matrices associated with
$\kappa_t$ and $\widetilde{\kappa}_0$.
Intuitively, $\mathcal{N}(f | \mu_c, \Sigma_c)$ can be viewed as the GP posterior 
where the prior has been ``swapped out'' \cite{neiswanger2017post}, and replaced with
$\text{GP}_0(f) / \widetilde{\text{GP}}_0(f)$.
Importantly, note that
$\lim_{\gamma \to 0} \widetilde{R}_t(f) = R_t(f)$,
 the prior-posterior-ratio (Eq.~\ref{eq:PPR}), with no restriction on $f$ or $\Fcal$. 
 
\paragraph{\emph{Remark:} the role of ``belief parameter'' $\gamma$.}
The parameter $\gamma$ plays important computational and statistical roles. Computationally speaking, numerical stability issues related to invertability are reduced by increasing $\gamma$.
Statistically, $\gamma$ implicitly defines the function class $\Fcal \equiv \Fcal_\gamma$ under consideration. $\gamma \to 0$ recovers an unrestricted $\Fcal_0$ that allows arbitrarily wiggly functions, and hence necessarily leads to large and pessimistic $C_t$. At the other extreme, $\gamma \to \infty$ recovers the usual posterior band used in BO, corresponding to the function class $\Fcal_\infty$ created with a full belief in $\GP_0$ (where complexity of a function can be thought of in terms of the mass assigned by the prior $\GP_0$).  To summarize, the ``belief parameter'' $\gamma$ plays three roles: 
\begin{itemize}\vspace{-0.05in}
\item[(A)] computational, providing numerical stability as $\gamma$ increases); 
\item[(B)] statistical, adding regularization that restricts the complexity of functions in $C_t$, and hence size of $C_t$, by implicitly defining $\Fcal_\gamma$); and 
\item[(C)] philosophical, trading a (Bayesian) subjective belief in the prior ($\gamma\to\infty)$ with (frequentist) robustness against  misspecification ($\gamma \to 0$).\vspace{-0.05in}
\end{itemize}
Returning to our simulation, the confidence sequence guarantees derived at $\gamma = 0$ provide robustness against \emph{arbitrary} misspecification of the prior, but our choice of $\gamma=10^{-2}$ seemed more reasonable if we think the prior is not completely ridiculous. An interesting direction for future work is to figure out how to automatically tune $\gamma$ in light of the aforementioned tradeoffs.

\paragraph{(3) Compute the confidence sequence.}
We can then use the confidence sequence
\begin{align*}
    C_t = \{ f \in \mathbb{R}^{|G_t|} : \widetilde{R}_t^{-1}(f) > \alpha \}.
\end{align*}

% \williex{attempt rewriting}
% \begin{align*}
%     C_t = \{ f \in \mathcal{F} : R_t(f) \leq \frac{1}{\alpha} \}.
% \end{align*}
% where $\mathcal{F} = \{f \in \mathbb{R}^{|G_t|} :
%         \frac{1}{\alpha} \text{GP}_t(f) \geq \widetilde{\text{GP}}_0\}$
% $= \{f \in \mathbb{R}^{|G_t|} :
%         \mathcal{L}_t(f) \geq \frac{\alpha}{Z} \frac{\widetilde{\text{GP}}_0(f)}{\text{GP}_0(f)} \}$

% \williex{
% Note to Aadi: we'll be computing one of these for each $X^\prime$ in our
% sequence, so we could write this as $C_t(X^\prime)$, and for that matter
% could denote the $R_t$, $\widetilde{R}_t$, and $G_t$ in terms of $X^\prime$
% as well if we wanted... what do you think?
% Also, is it correct to call this a confidence sequence?
% Or maybe we should call it a ``[something] for $X^\prime$''.
% }

Thus we know that $C_t$ is an ellipsoid defined by the superlevel set of
$\widetilde{R}_t^{-1}(f)$. To compute $C_t$, we can traverse outwards from the 
posterior-prior ratio mean $\mu_c$ until we have found the Mahalanobis distance $k$ to the
isocontour $\mathcal{I} = \{f \in \mathbb{R}^{|G_t|} : c \mathcal{N}(f | \mu_c, \Sigma_c) = \alpha \}$.
    
% We can determine the number of standard deviations $k$ that $\mathcal{C}$ 
% is away from the mean of $\mu_c$.
We can therefore view $C_t$ as the $k$-sigma ellipsoid of the posterior 
GP (normal distribution) given by  $\mathcal{N}(f | \mu_c, \sigma_c))$. Using this
confidence ellipsoid over $f$, we can compute a lower confidence bound for the value of
$f(X^\prime)$, which we use as a LCB-style acquisition function
$\varphi_t(x)$ at input $X^\prime$.

To summarize the detailed explanations, our simulations use:
\begin{align*}
    \widetilde{R}_t(f) = \frac{\widetilde{\text{GP}}_0(f)}{\text{GP}_t(f)}
    = \frac{\text{GP}_0(f)}{\text{GP}_t(f)} \frac{\widetilde{\text{GP}}_0(f)}{\text{GP}_0(f)}
    = R_t(f) \frac{\widetilde{\text{GP}}_0(f)}{\text{GP}_0(f)},
\end{align*}
where $\widetilde{\text{GP}}_0(f)$ is the same as $\text{GP}_0(f)$, except with the signal
variance parameter $\sigma^2$  set to $\sigma^2 (1 + \gamma)$. 
% So when 
% $\gamma \rightarrow 0$, $\widetilde{\text{GP}}_0(f) \rightarrow \text{GP}_0(f)$,
% and $\widetilde{R}_t(f) \rightarrow R_t(f)$.

% \paragraph{BO simulations.}
\paragraph{BO simulations: GP-LCB versus CS-LCB.}
We demonstrate BO using $C_t$ (following the procedure outlined above, which we call CS-LCB)
and compare it against the GP-LCB algorithm. Results for these experiments are given in Appendix~\ref{sec:appendix}.
Briefly, we applied these methods to optimize an unknown function $f^*$ in both the well-specified
and misspecified settings. The findings were as expected: under a misspecified prior,
GP-LCB is overconfident about its progress and fails to minimize $f^*$, while CS-LCB
mitigates the issue. For a well-specified prior, both algorithms find the minimizer,
but GP-LCB finds it sooner than CS-LCB. 

\paragraph{Robustness to misspecified likelihood.} Throughout this paper, we have assumed correctness of the likelihood model \eqref{eq:model}, but what if that assumption is suspect? In the supplement, we repeat the experiment in Figure~\ref{fig:viz}, except when the true noise $\eta^*$ is half the value $\eta$ used by the working likelihood (Figure~\ref{fig:viz-lownoise}), as well as when $\eta^*$ is double of $\eta$ (Figure~\ref{fig:viz-highnoise}).
    As expected, when the noise is smaller than anticipated, our CS remains robust to the prior misspecification, but when the noise is larger, we begin to notice failures in our CS. We propose a simple fix: define $\check R_t := R_t^\beta$, for some $\beta \in (0,1)$, and construct the CS based on $\check R_t$. Figure~\ref{fig:viz-powered} uses $\beta=0.75$ and reports promising results. This procedure is inspired by a long line of work in Bayesian inference that proposes raising likelihoods to a power less than one in order to increase robustness \cite{ibrahim2000power,royall2003interpreting,grunwald2012safe,grunwald2017inconsistency,miller2019robust,wasserman2019universal}. Since we desire frequentist coverage guarantees for a Bayesian working model (not assuming correctness of a Bayesian prior), we simply point out that $\check R_t$ is not a martingale like $R_t$, and is instead a supermartingale due to Jensen's inequality. Since Ville's inequality applies, the resulting CS is still valid. Thus it appears at first glance, that one can obtain some amount of robustness against both misspecified priors and likelihoods. 
    However, as mentioned below, merging this idea with hyperparameter tuning and a data-dependent choice of $\beta$ seems critical for practice.

% \vspace{-1mm}
\section{Discussion}
\label{sec:discussion}
% \vspace{-2mm}

Confidence sequences were introduced and studied in depth by Robbins along with Darling, Siegmund and Lai \cite{darling_confidence_1967,robbins_boundary_1970,lai_boundary_1976,lai_confidence_1976}. The topic was subsequently somewhat dormant but came back into vogue due to applications to best-arm identification in multi-armed bandits \cite{jamieson_lil_2014}. Techniques related to nonnegative supermartingales, the mixture method, Ville's inequality, and nonparametric confidence sequences have been studied very recently --- see \citet{howard_exponential_2018,howard_uniform_2019,kaufmann2018mixture,howard2019sequential,waudby2020bounded,waudby2020confidence} and references therein. 
% Confidence sequences are dual to sequential hypothesis tests and anytime $p$-values, ideas dating back to Wald \cite{wald_sequential_1947}. 
They are closely tied to optional stopping, continuous monitoring of experiments and scientific reproducibility \cite{wald_sequential_1947,balsubramani_sharp_2014,balsubramani_sequential_2016,johari_peeking_2017,shafer_test_2011,grunwald_safe_2019,howard_uniform_2019}. We are unaware of other work that utilizes them to quantify uncertainty in a BO context.

Many important open questions remain. We describe three directions:\vspace{-0.05in}
\begin{itemize}[leftmargin=7mm]
    \item \textbf{Hyperparameter tuning.} It is common in BO practice to tune hyperparameters on the fly \cite{snoek2012practical, shahriari2015taking, kandasamy2019tuning, neiswanger2019probo}. These can alleviate some problems mentioned in the first page of this paper, but probably only if the kernel is a good match and the function has homogeneous smoothness. We would like to explore if hyperparameter tuning can be integrated into confidence sequences. 
    
The manner in which we estimate hyperparameters is critical, as highlighted by the recent work of \citet{bachoc2018asymptotic} who asks: what happens when we estimate hyperparameters of our kernel using (A) maximum likelihood estimation, or (B) cross-validation, when restricting our attention to some prespecified set of hyperparameters which do not actually capture the true covariance function? The answer turns out to be subtle: the Maximum Likelihood estimator asymptotically minimizes
a Kullback-Leibler divergence to the misspecified parametric set, while Cross Validation asymptotically
minimizes the integrated square prediction error; Bachoc demonstrates that the two approaches could be rather different in practice.

    \item \textbf{The belief parameter $\gamma$.} Can $\gamma$ be tuned automatically, or updated in a data-dependent way? Further, if we move to the aforementioned hyperparameter tuning setup, can we design a belief parameter $\gamma$ that can smoothly trade off our belief in the tuned prior against robustness to misspecification? Perhaps we would want $\gamma \to \infty$ with sample size so that as we get more data to tune our priors better, we would need less robustness protection. Further, perhaps we may wish to use a convex combination of kernels, with a weight of $1/(1+\gamma)$ for a simpler kernel (like Gaussian) and a weight of $\gamma/(1+\gamma)$ for a more complex kernel, so that as $\gamma \to \infty$, we not only have more faith in our prior, but we may also allow more complex functions.
    \item \textbf{Computationally tractable choices for $\Fcal$.} While the method introduced in Section~\ref{sec:mainresult} is general, some care had to be taken when instantiating it in the experiments of Section~\ref{sec:simulations}, because the choice of function class $\Fcal$ had to be chosen to make computation of the set $C_t$ easy. Can we expand the set of computational tools so that these ideas are applicable for other choices of $\Fcal$? How do we scale these methods to work in high dimensions?
\end{itemize} 
The long-term utility of our new ideas will rely on finding suitable answers to the above questions. There are other recent works that study the mean-squared error of GPs under prior misspecification~\cite{beckers_mse_2018}, or under potentially adversarial noise in the observation model~\cite{bogunovic2020corruption}. Their goals are orthogonal to ours (uncertainty quantification), but a cross-pollination of ideas may be beneficial to both efforts.

\smallskip

We end with a cautionary quote from Freedman's Wald lecture~\cite{freedman1999wald}:
\begin{quote}
\textit{With a large sample from a smooth, finite-dimensional statistical model, the Bayes estimate and the maximum likelihood estimate will be close. Furthermore, the posterior distribution of the parameter vector
around the posterior mean must be close to the distribution of the maximum
likelihood estimate around truth: both are asymptotically normal with mean
0, and both have the same asymptotic covariance matrix. That is the con-
tent of the Bernstein–von Mises theorem. Thus, a Bayesian 95\%-confidence
set must have frequentist coverage of about 95\%, and conversely. In particular, Bayesians and frequentists are free to use each other’s confidence sets. However, even for the simplest infinite-dimensional models, the Bernstein–von Mises theorem does not hold (see Cox~\cite{cox1993analysis})...The sad lesson for inference is this. If frequentist coverage probabilities are wanted in an infinite-dimensional problem, then frequentist coverage probabilities must be computed. Bayesians, too, need to proceed with caution in the infinite-dimensional case, unless they are convinced of the fine details of their priors. Indeed, the consistency of their estimates and the coverage probability of their confidence sets depend on the details of their priors. 
}
\end{quote}

Our experiments match the expectations set by the above quote: while the practical appeal of Bayesian credible posterior GP intervals is apparent---they are easy to calculate and visualize---they appear to be inconsistent under even minor prior misspecification (Figure~\ref{fig:viz}), and this is certainly seems to be an infinite-dimensional issue. It is perhaps related to the fact that there is no analog of the Lebesgue measure in infinite dimensions, and thus our finite-dimensional intuition that ``any Gaussian prior puts nonzero mass everywhere'' does not seem to be an accurate intuition in infinite dimensions.

\subsection*{Acknowledgments}
AR thanks Akshay Balsubramani for related conversations.
AR acknowledges funding from an Adobe Faculty Research Award, and an NSF DMS 1916320 grant.
WN was supported by U.S. Department of Energy Office of Science under Contract No. DE-AC02-76SF00515.

\bibliography{main}

\begin{thebibliography}{41}
\providecommand{\natexlab}[1]{#1}
\providecommand{\url}[1]{\texttt{#1}}
\expandafter\ifx\csname urlstyle\endcsname\relax
  \providecommand{\doi}[1]{doi: #1}\else
  \providecommand{\doi}{doi: \begingroup \urlstyle{rm}\Url}\fi

\bibitem[Bachoc(2018)]{bachoc2018asymptotic}
Fran{\c{c}}ois Bachoc.
\newblock Asymptotic analysis of covariance parameter estimation for gaussian
  processes in the misspecified case.
\newblock \emph{Bernoulli}, 24\penalty0 (2):\penalty0 1531--1575, 2018.

\bibitem[Balsubramani(2014)]{balsubramani_sharp_2014}
Akshay Balsubramani.
\newblock Sharp finite-time iterated-logarithm martingale concentration.
\newblock \emph{arXiv preprint, arXiv:1405.2639}, 2014.

\bibitem[Balsubramani and Ramdas(2016)]{balsubramani_sequential_2016}
Akshay Balsubramani and Aaditya Ramdas.
\newblock Sequential nonparametric testing with the law of the iterated
  logarithm.
\newblock In \emph{Proceedings of the {Thirty}-{Second} {Conference} on
  {Uncertainty} in {Artificial} {Intelligence}}, 2016.

\bibitem[{Beckers} et~al.(2018){Beckers}, {Umlauft}, and
  {Hirche}]{beckers_mse_2018}
T.~{Beckers}, J.~{Umlauft}, and S.~{Hirche}.
\newblock Mean square prediction error of misspecified gaussian process models.
\newblock In \emph{IEEE Conference on Decision and Control (CDC)}, 2018.

\bibitem[Bect et~al.(2019)Bect, Bachoc, and
  Ginsbourger]{bect2019supermartingale}
Julien Bect, Fran{\c{c}}ois Bachoc, and David Ginsbourger.
\newblock A supermartingale approach to {G}aussian process based sequential
  design of experiments.
\newblock \emph{Bernoulli}, 25\penalty0 (4A):\penalty0 2883--2919, 2019.

\bibitem[Bogunovic et~al.(2020)Bogunovic, Krause, and
  Jonathan]{bogunovic2020corruption}
Ilija Bogunovic, Andreas Krause, and Scarlett Jonathan.
\newblock Corruption-tolerant {G}aussian process bandit optimization.
\newblock In \emph{International Conference on Artificial Intelligence and
  Statistics (AISTATS)}, 2020.

\bibitem[Cox(1993)]{cox1993analysis}
Dennis~D Cox.
\newblock An analysis of {B}ayesian inference for nonparametric regression.
\newblock \emph{The Annals of Statistics}, pages 903--923, 1993.

\bibitem[Darling and Robbins(1967)]{darling_confidence_1967}
Donald~A. Darling and Herbert Robbins.
\newblock Confidence sequences for mean, variance, and median.
\newblock \emph{Proceedings of the National Academy of Sciences}, 58\penalty0
  (1):\penalty0 66--68, 1967.

\bibitem[Freedman(1999)]{freedman1999wald}
David Freedman.
\newblock Wald {L}ecture: On the {B}ernstein-von {M}ises theorem with
  infinite-dimensional parameters.
\newblock \emph{The Annals of Statistics}, 27\penalty0 (4):\penalty0
  1119--1141, 1999.

\bibitem[Gr{\"u}nwald(2012)]{grunwald2012safe}
Peter Gr{\"u}nwald.
\newblock The safe {B}ayesian.
\newblock In \emph{International Conference on Algorithmic Learning Theory},
  pages 169--183. Springer, 2012.

\bibitem[Gr{\"u}nwald and Van~Ommen(2017)]{grunwald2017inconsistency}
Peter Gr{\"u}nwald and Thijs Van~Ommen.
\newblock Inconsistency of {B}ayesian inference for misspecified linear models,
  and a proposal for repairing it.
\newblock \emph{Bayesian Analysis}, 12\penalty0 (4):\penalty0 1069--1103, 2017.

\bibitem[Grünwald et~al.(2019)Grünwald, de~Heide, and
  Koolen]{grunwald_safe_2019}
Peter Grünwald, Rianne de~Heide, and Wouter Koolen.
\newblock Safe {testing}.
\newblock \emph{arXiv:1906.07801}, June 2019.

\bibitem[Howard and Ramdas(2019)]{howard2019sequential}
Steven~R Howard and Aaditya Ramdas.
\newblock Sequential estimation of quantiles with applications to
  {A}/{B}-testing and best-arm identification.
\newblock \emph{arXiv preprint arXiv:1906.09712}, 2019.

\bibitem[Howard et~al.(2020)Howard, Ramdas, McAuliffe, and
  Sekhon]{howard_exponential_2018}
Steven~R Howard, Aaditya Ramdas, Jon McAuliffe, and Jasjeet Sekhon.
\newblock Time-uniform {Chernoff} bounds via nonnegative supermartingales.
\newblock \emph{Probability Surveys}, 17:\penalty0 257--317, 2020.

\bibitem[Howard et~al.(2021)Howard, Ramdas, McAuliffe, and
  Sekhon]{howard_uniform_2019}
Steven~R Howard, Aaditya Ramdas, Jon McAuliffe, and Jasjeet Sekhon.
\newblock Time-uniform, nonparametric, nonasymptotic confidence sequences.
\newblock \emph{The Annals of Statistics}, 2021.

\bibitem[Ibrahim and Chen(2000)]{ibrahim2000power}
Joseph~G Ibrahim and Ming-Hui Chen.
\newblock Power prior distributions for regression models.
\newblock \emph{Statistical Science}, 15\penalty0 (1):\penalty0 46--60, 2000.

\bibitem[Jamieson et~al.(2014)Jamieson, Malloy, Nowak, and
  Bubeck]{jamieson_lil_2014}
Kevin Jamieson, Matthew Malloy, Robert Nowak, and Sébastien Bubeck.
\newblock lil' {UCB}: An optimal exploration algorithm for multi-armed bandits.
\newblock In \emph{Proceedings of {The} 27th {Conference} on {Learning}
  {Theory}}, volume~35, pages 423--439, 2014.

\bibitem[Johari et~al.(2017)Johari, Koomen, Pekelis, and
  Walsh]{johari_peeking_2017}
Ramesh Johari, Pete Koomen, Leonid Pekelis, and David Walsh.
\newblock Peeking at {A}/{B} tests: Why it matters, and what to do about it.
\newblock In \emph{Proceedings of the 23rd ACM SIGKDD International Conference
  on Knowledge Discovery and Data Mining}, pages 1517--1525, 2017.

\bibitem[Kandasamy et~al.(2020)Kandasamy, Vysyaraju, Neiswanger, Paria,
  Collins, Schneider, Poczos, and Xing]{kandasamy2019tuning}
Kirthevasan Kandasamy, Karun~Raju Vysyaraju, Willie Neiswanger, Biswajit Paria,
  Christopher~R Collins, Jeff Schneider, Barnabas Poczos, and Eric~P Xing.
\newblock Tuning hyperparameters without grad students: Scalable and robust
  {B}ayesian optimisation with {D}ragonfly.
\newblock \emph{Journal of Machine Learning Research}, 21\penalty0
  (81):\penalty0 1--27, 2020.

\bibitem[Kaufmann and Koolen(2018)]{kaufmann2018mixture}
Emilie Kaufmann and Wouter Koolen.
\newblock Mixture martingales revisited with applications to sequential tests
  and confidence intervals.
\newblock \emph{arXiv:1811.11419}, 2018.

\bibitem[Lai(1976{\natexlab{a}})]{lai_boundary_1976}
Tze~Leung Lai.
\newblock Boundary crossing probabilities for sample sums and confidence
  sequences.
\newblock \emph{The Annals of Probability}, 4\penalty0 (2):\penalty0 299--312,
  1976{\natexlab{a}}.

\bibitem[Lai(1976{\natexlab{b}})]{lai_confidence_1976}
Tze~Leung Lai.
\newblock On {Confidence} {Sequences}.
\newblock \emph{The Annals of Statistics}, 4\penalty0 (2):\penalty0 265--280,
  1976{\natexlab{b}}.

\bibitem[Miller and Dunson(2019)]{miller2019robust}
Jeffrey~W Miller and David~B Dunson.
\newblock Robust {B}ayesian inference via coarsening.
\newblock \emph{Journal of the American Statistical Association}, 114\penalty0
  (527):\penalty0 1113--1125, 2019.

\bibitem[Mockus et~al.(1978)Mockus, Tiesis, and
  Zilinskas]{mockus1978application}
Jonas Mockus, Vytautas Tiesis, and Antanas Zilinskas.
\newblock The application of bayesian methods for seeking the extremum.
\newblock \emph{Towards global optimization}, 2\penalty0 (117-129):\penalty0 2,
  1978.

\bibitem[Neiswanger and Xing(2017)]{neiswanger2017post}
Willie Neiswanger and Eric Xing.
\newblock Post-inference prior swapping.
\newblock In \emph{Proceedings of the 34th International Conference on Machine
  Learning-Volume 70}, pages 2594--2602. JMLR. org, 2017.

\bibitem[Neiswanger et~al.(2019)Neiswanger, Kandasamy, Poczos, Schneider, and
  Xing]{neiswanger2019probo}
Willie Neiswanger, Kirthevasan Kandasamy, Barnabas Poczos, Jeff Schneider, and
  Eric Xing.
\newblock Probo: a framework for using probabilistic programming in {B}ayesian
  optimization.
\newblock \emph{arXiv preprint arXiv:1901.11515}, 2019.

\bibitem[Robbins and Siegmund(1970)]{robbins_boundary_1970}
Herbert Robbins and David Siegmund.
\newblock Boundary {crossing} {probabilities} for the {Wiener} {process} and
  {sample} {sums}.
\newblock \emph{The Annals of Mathematical Statistics}, 41\penalty0
  (5):\penalty0 1410--1429, 1970.

\bibitem[Royall and Tsou(2003)]{royall2003interpreting}
Richard Royall and Tsung-Shan Tsou.
\newblock Interpreting statistical evidence by using imperfect models: robust
  adjusted likelihood functions.
\newblock \emph{Journal of the Royal Statistical Society: Series B (Statistical
  Methodology)}, 65\penalty0 (2):\penalty0 391--404, 2003.

\bibitem[Schulz et~al.(2016)Schulz, Speekenbrink, Hern{\'a}ndez-Lobato,
  Ghahramani, and Gershman]{schulz2016quantifying}
Eric Schulz, Maarten Speekenbrink, Jos{\'e}~M Hern{\'a}ndez-Lobato, Zoubin
  Ghahramani, and Samuel~J Gershman.
\newblock Quantifying mismatch in {B}ayesian optimization.
\newblock In \emph{NIPS workshop on Bayesian optimization: Black-box
  optimization and beyond}, 2016.

\bibitem[Shafer and Vovk(2019)]{shafer2019game}
Glenn Shafer and Vladimir Vovk.
\newblock \emph{Game-Theoretic Foundations for Probability and Finance}, volume
  455.
\newblock John Wiley \& Sons, 2019.

\bibitem[Shafer et~al.(2011)Shafer, Shen, Vereshchagin, and
  Vovk]{shafer_test_2011}
Glenn Shafer, Alexander Shen, Nikolai Vereshchagin, and Vladimir Vovk.
\newblock Test martingales, {B}ayes factors and $p$-values.
\newblock \emph{Statistical Science}, 26\penalty0 (1):\penalty0 84--101, 2011.

\bibitem[Shahriari et~al.(2015)Shahriari, Swersky, Wang, Adams, and
  De~Freitas]{shahriari2015taking}
Bobak Shahriari, Kevin Swersky, Ziyu Wang, Ryan~P Adams, and Nando De~Freitas.
\newblock Taking the human out of the loop: A review of {B}ayesian
  optimization.
\newblock \emph{Proceedings of the IEEE}, 104\penalty0 (1):\penalty0 148--175,
  2015.

\bibitem[Snoek et~al.(2012)Snoek, Larochelle, and Adams]{snoek2012practical}
Jasper Snoek, Hugo Larochelle, and Ryan~P Adams.
\newblock Practical {B}ayesian optimization of machine learning algorithms.
\newblock In \emph{Advances in Neural Information Processing Systems}, pages
  2951--2959, 2012.

\bibitem[Sollich(2002)]{sollich2002gaussian}
Peter Sollich.
\newblock Gaussian process regression with mismatched models.
\newblock In \emph{Advances in Neural Information Processing Systems}, pages
  519--526, 2002.

\bibitem[Srinivas et~al.(2010)Srinivas, Krause, Kakade, and
  Seeger]{srinivas2009gaussian}
Niranjan Srinivas, Andreas Krause, Sham Kakade, and Matthias Seeger.
\newblock Gaussian process optimization in the bandit setting: no regret and
  experimental design.
\newblock In \emph{Proceedings of the 27th International Conference on Machine
  Learning}, pages 1015--1022, 2010.

\bibitem[Ville(1939)]{ville_etude_1939}
J~Ville.
\newblock \emph{Étude {Critique} de la {Notion} de {Collectif} (PhD Thesis)}.
\newblock Gauthier-Villars, Paris, 1939.

\bibitem[Wald(1947)]{wald_sequential_1947}
Abraham Wald.
\newblock \emph{Sequential {Analysis}}.
\newblock John Wiley \& Sons, New York, 1947.

\bibitem[Wasserman et~al.(2020)Wasserman, Ramdas, and
  Balakrishnan]{wasserman2019universal}
Larry Wasserman, Aaditya Ramdas, and Sivaraman Balakrishnan.
\newblock Universal inference.
\newblock \emph{Proceedings of the National Academy of Sciences}, 2020.

\bibitem[Waudby-Smith and Ramdas(2020{\natexlab{a}})]{waudby2020bounded}
Ian Waudby-Smith and Aaditya Ramdas.
\newblock Estimating means of bounded random variables by betting.
\newblock \emph{arXiv preprint arXiv:2010.09686}, 2020{\natexlab{a}}.

\bibitem[Waudby-Smith and Ramdas(2020{\natexlab{b}})]{waudby2020confidence}
Ian Waudby-Smith and Aaditya Ramdas.
\newblock Confidence sequences for sampling without replacement.
\newblock \emph{Advances in Neural Information Processing Systems}, 33,
  2020{\natexlab{b}}.

\bibitem[Williams and Rasmussen(2006)]{williams2006gaussian}
Christopher~KI Williams and Carl~E Rasmussen.
\newblock \emph{Gaussian processes for machine learning}, volume~2.
\newblock MIT press Cambridge, MA, 2006.

\end{thebibliography}

\newpage
\appendix

\section{Bayesian Optimization Simulations}
\label{sec:appendix}

We demonstrate BO using our confidence sequence $C_t$
(following the procedure outlined in Section~\ref{sec:numericaltools})
and compare it against the GP-LCB algorithm.
Results for these experiments are shown below, where we apply these methods to
optimize a function $f$ in both the misspecified
prior (Figure~\ref{fig:bo_1}) and correctly specified prior
(Figure~\ref{fig:bo_2}) settings.
We find that under a misspecified prior,
GP-LCB can yield inaccurate confidence bands and fail to find the
optimum of $f$, while BO using $C_t$ (CS-LCB) can help mitigate
this issue.

\begin{figure}[H]
\centering

\begin{subfigure}{0.45\textwidth}
  \includegraphics[width=\linewidth]{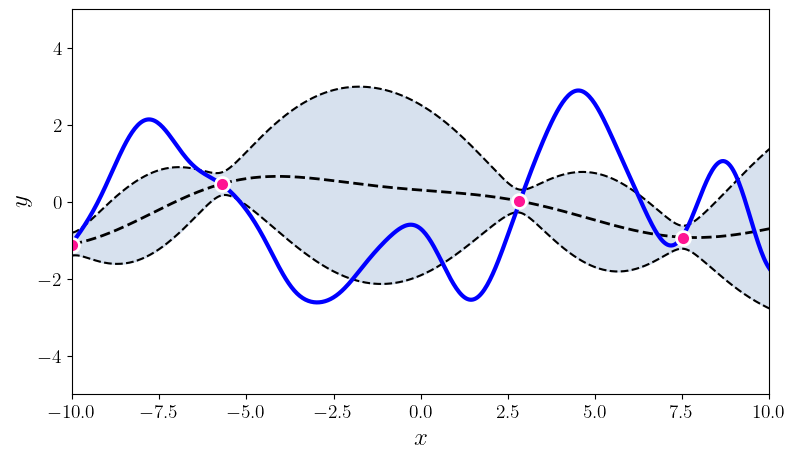}
%   \caption{Model prior}
\end{subfigure}\hfil
\begin{subfigure}{0.45\textwidth}
  \includegraphics[width=\linewidth]{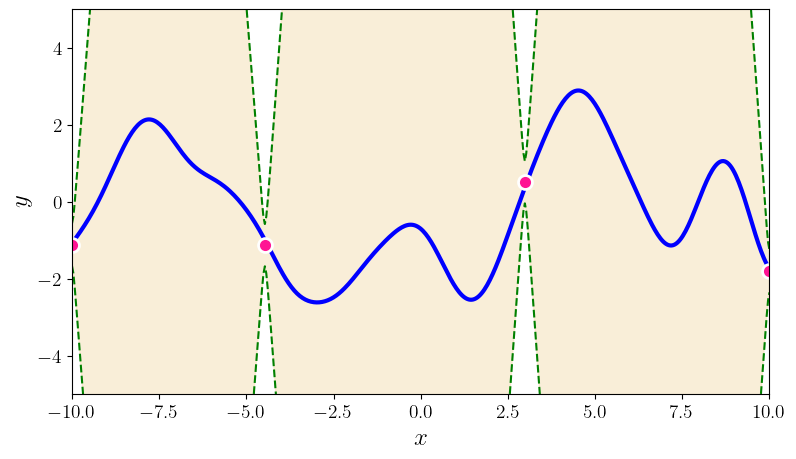}
%   \caption{True prior}
\end{subfigure}
\\
\begin{subfigure}{0.45\textwidth}
  \includegraphics[width=\linewidth]{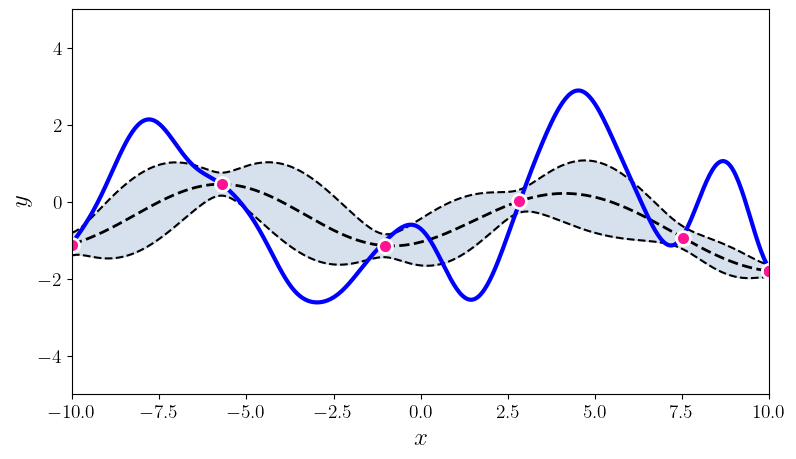}
%   \caption{$n=3$}
\end{subfigure}\hfil
\begin{subfigure}{0.45\textwidth}
  \includegraphics[width=\linewidth]{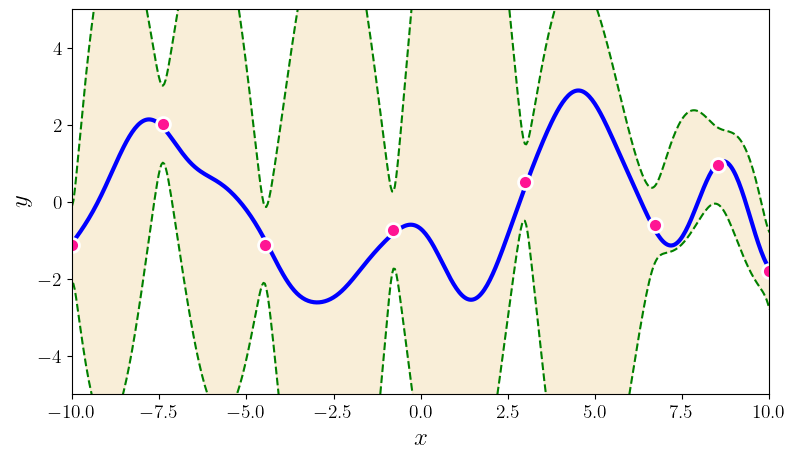}
%   \caption{$n=5$}
\end{subfigure}
\\
\begin{subfigure}{0.45\textwidth}
  \includegraphics[width=\linewidth]{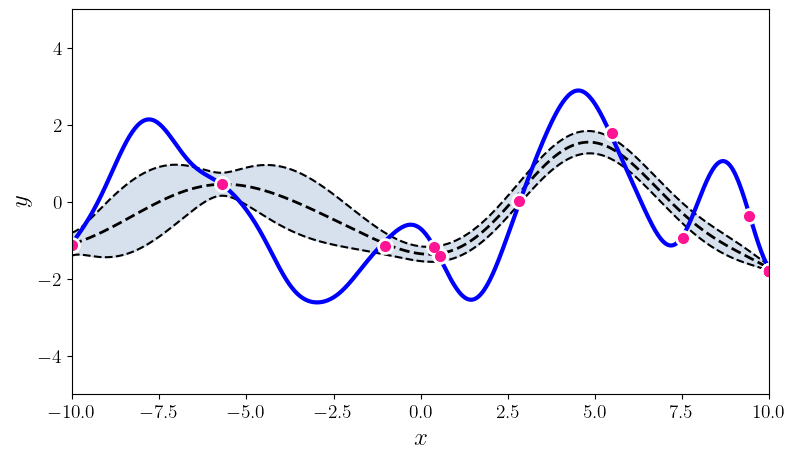}
%   \caption{$n=15$}
\end{subfigure}\hfil
\begin{subfigure}{0.45\textwidth}
  \includegraphics[width=\linewidth]{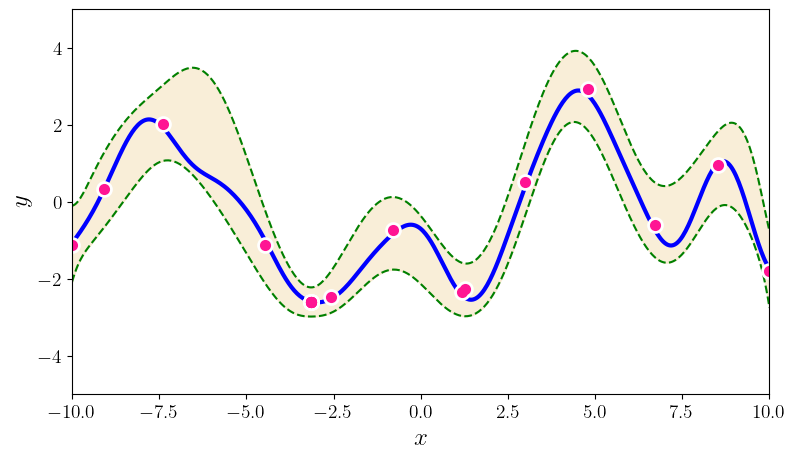}
%   \caption{$n=17$}
\end{subfigure}
\\
\begin{subfigure}{0.45\textwidth}
  \includegraphics[width=\linewidth]{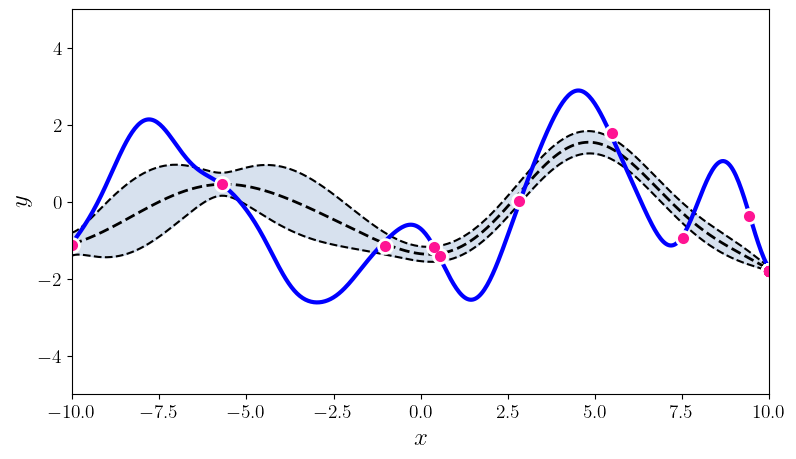}
%   \caption{$n=25$}
\end{subfigure}\hfil
\begin{subfigure}{0.45\textwidth}
  \includegraphics[width=\linewidth]{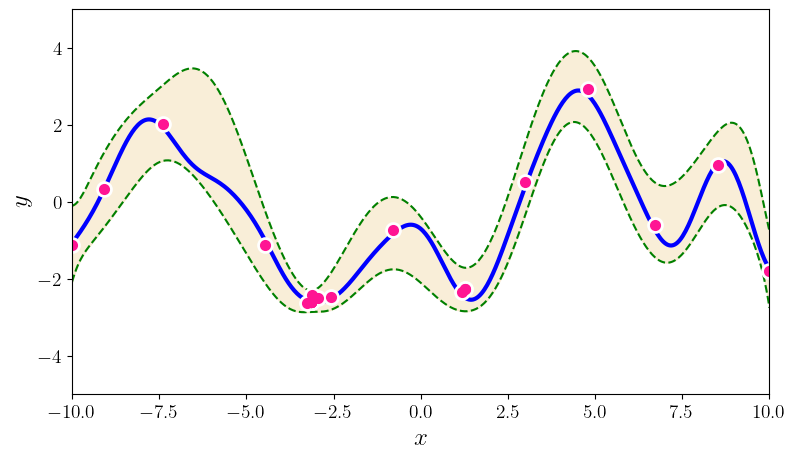}
%   \caption{$n=40$}
\end{subfigure}

\caption{
This figure shows GP-LCB (left column) and CS-LCB (right column) for a misspecified prior, showing $t = 3, 7, 18, 25$
(rows 1-4). Here, GP-LCB yields inaccurate confidence bands, repeatedly queries at the wrong point (around $x=10.0$), and fails to find
the minimizer of $f$, while CS-LCB successfully finds the minimizer
(around $x = -3.0$).
}
\label{fig:bo_1}

\end{figure}

\begin{figure}[H]
\centering

\begin{subfigure}{0.45\textwidth}
  \includegraphics[width=\linewidth]{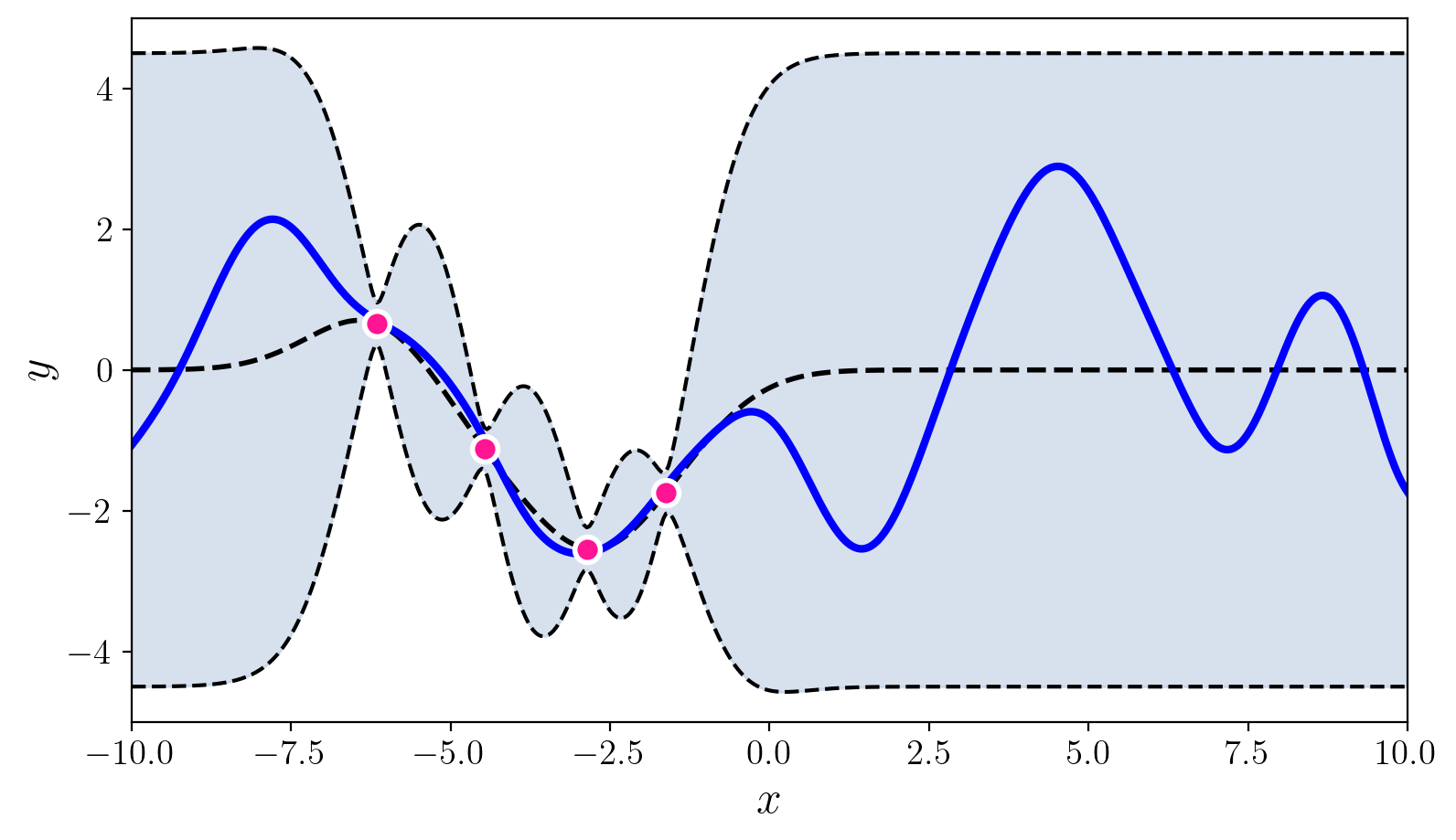}
%   \caption{Model prior}
\end{subfigure}\hfil
\begin{subfigure}{0.45\textwidth}
  \includegraphics[width=\linewidth]{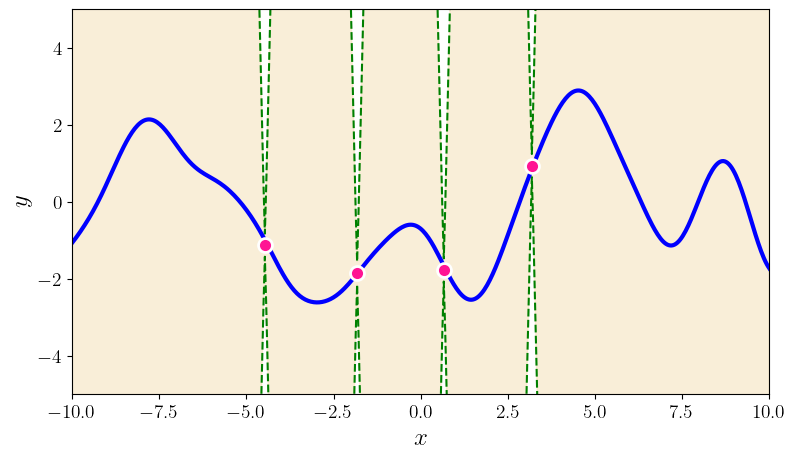}
%   \caption{True prior}
\end{subfigure}
\\
\begin{subfigure}{0.45\textwidth}
  \includegraphics[width=\linewidth]{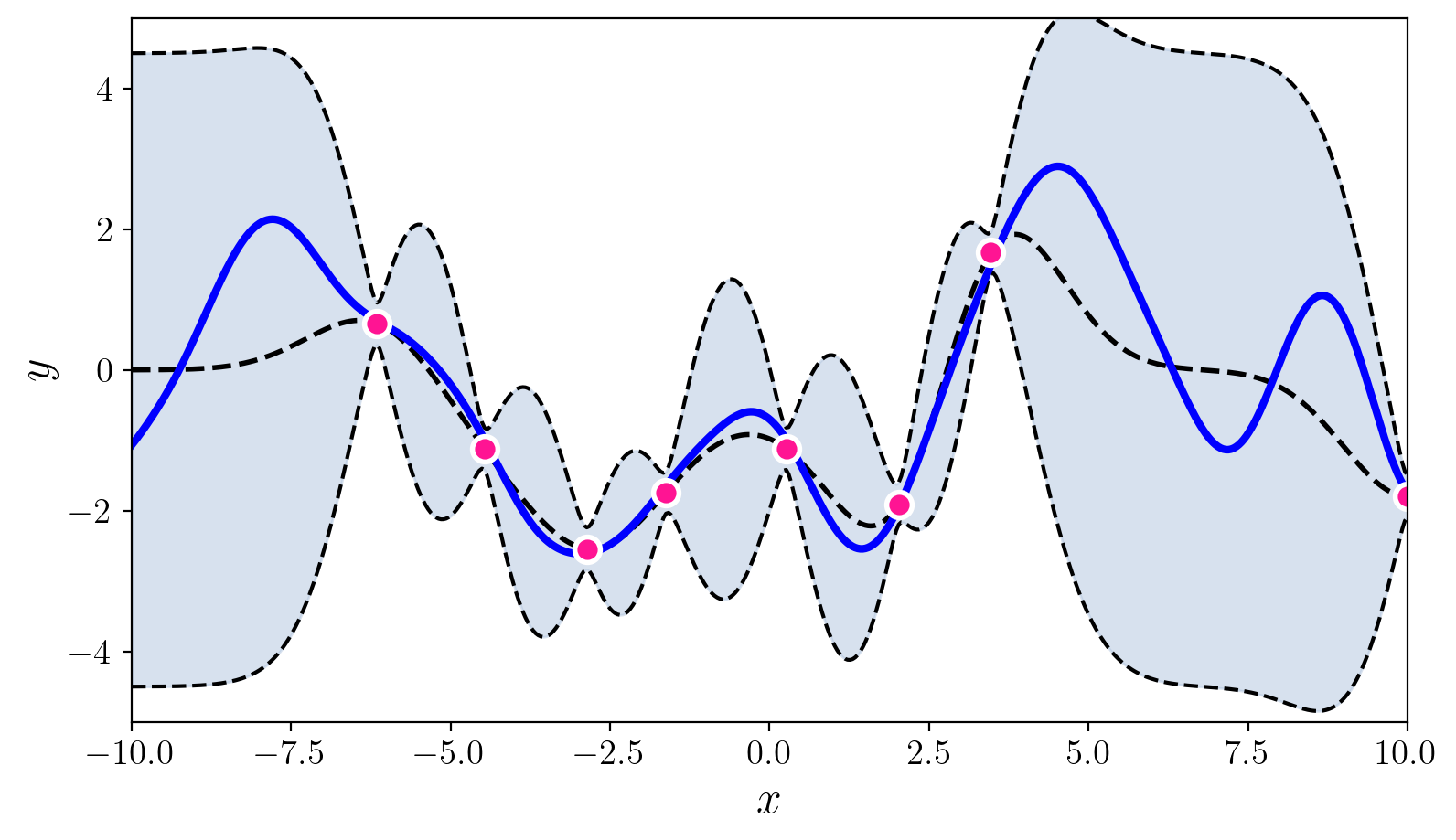}
%   \caption{$n=3$}
\end{subfigure}\hfil
\begin{subfigure}{0.45\textwidth}
  \includegraphics[width=\linewidth]{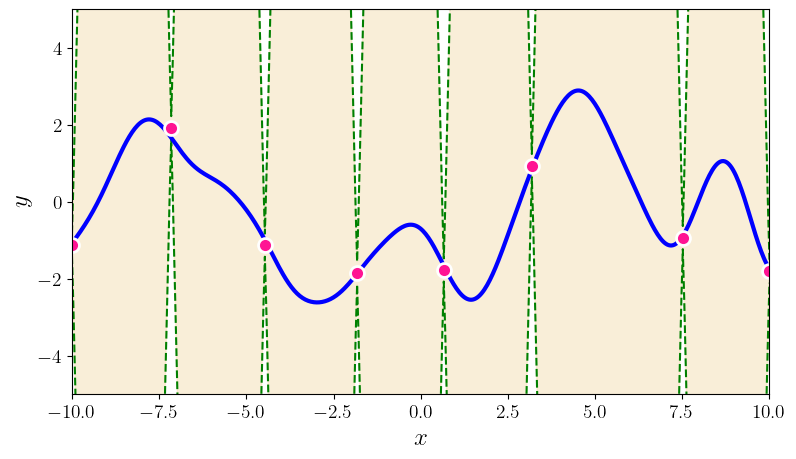}
%   \caption{$n=5$}
\end{subfigure}
\\
\begin{subfigure}{0.45\textwidth}
  \includegraphics[width=\linewidth]{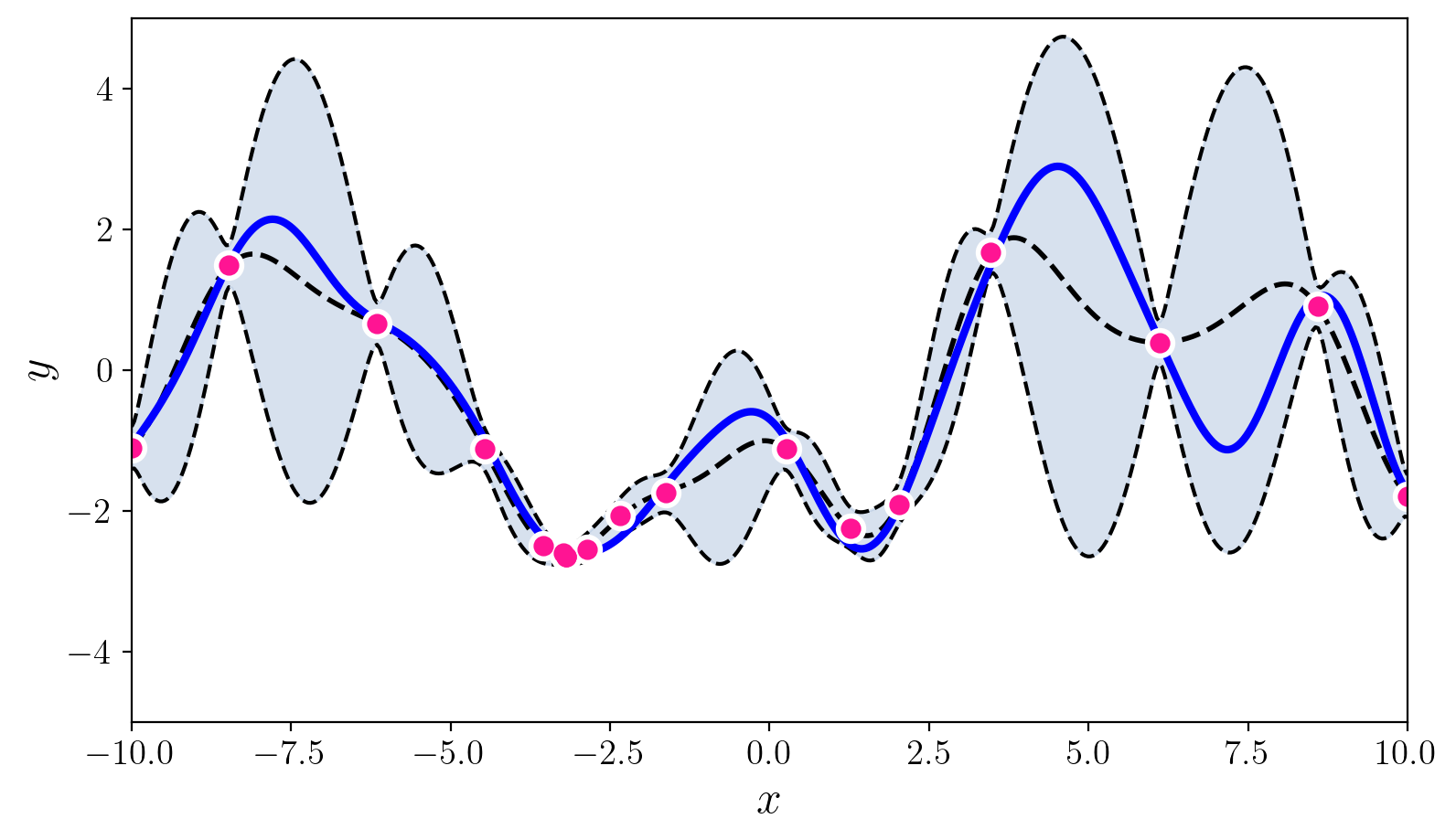}
%   \caption{$n=15$}
\end{subfigure}\hfil
\begin{subfigure}{0.45\textwidth}
  \includegraphics[width=\linewidth]{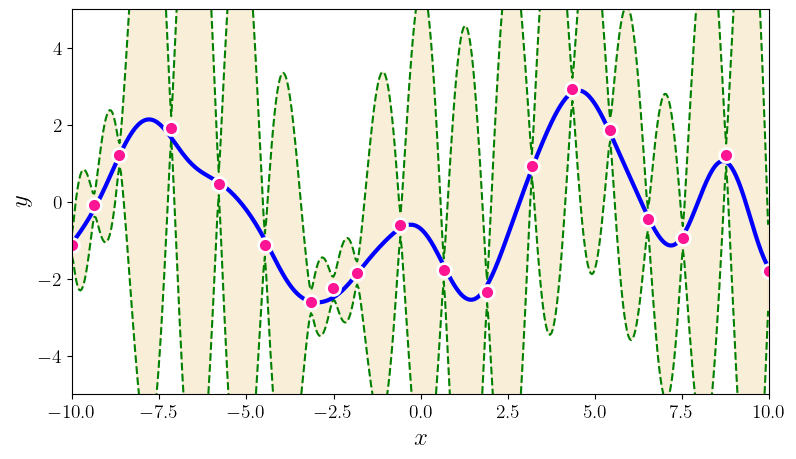}
%   \caption{$n=17$}
\end{subfigure}
\\
\begin{subfigure}{0.45\textwidth}
  \includegraphics[width=\linewidth]{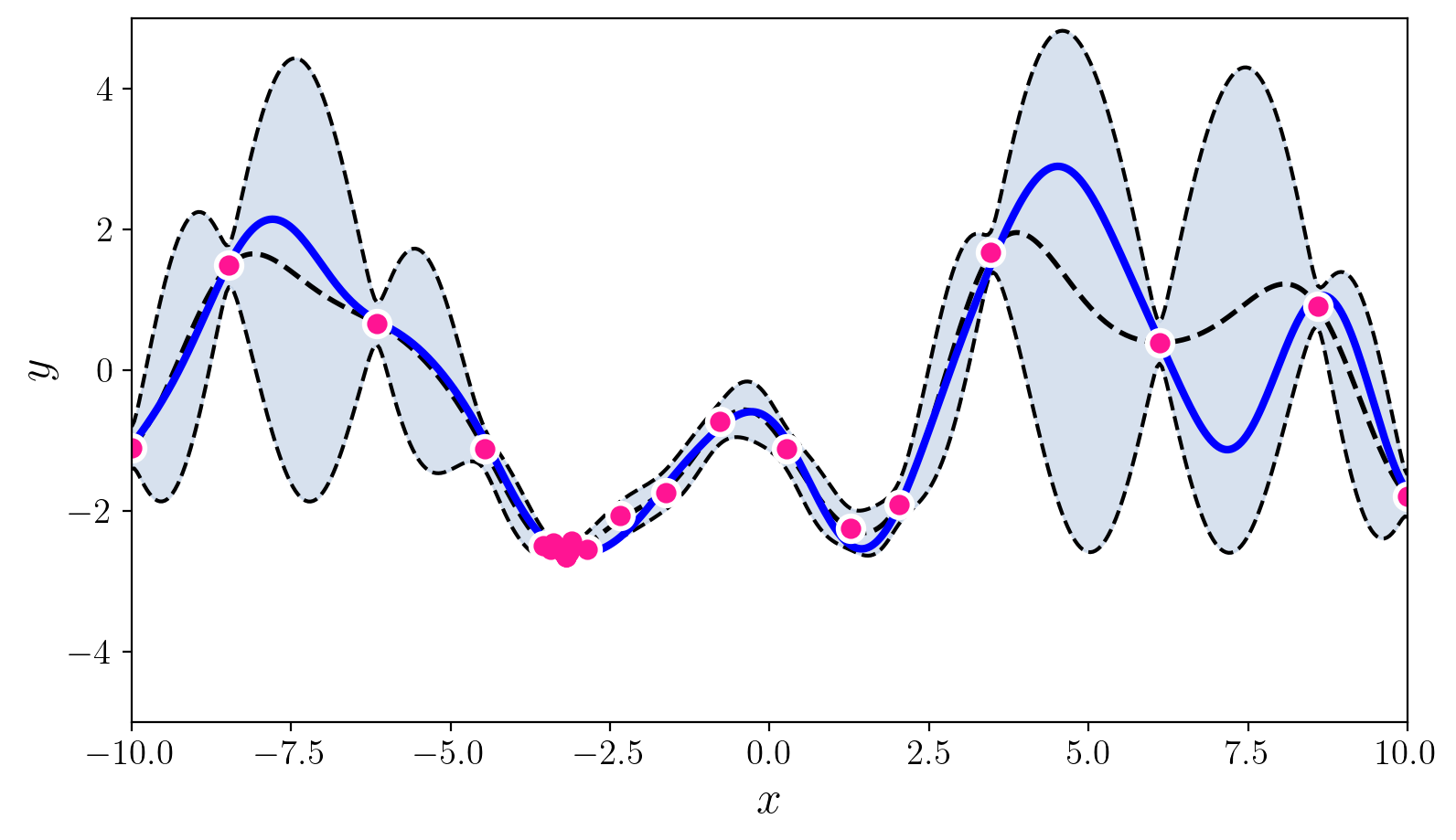}
%   \caption{$n=25$}
\end{subfigure}\hfil
\begin{subfigure}{0.45\textwidth}
  \includegraphics[width=\linewidth]{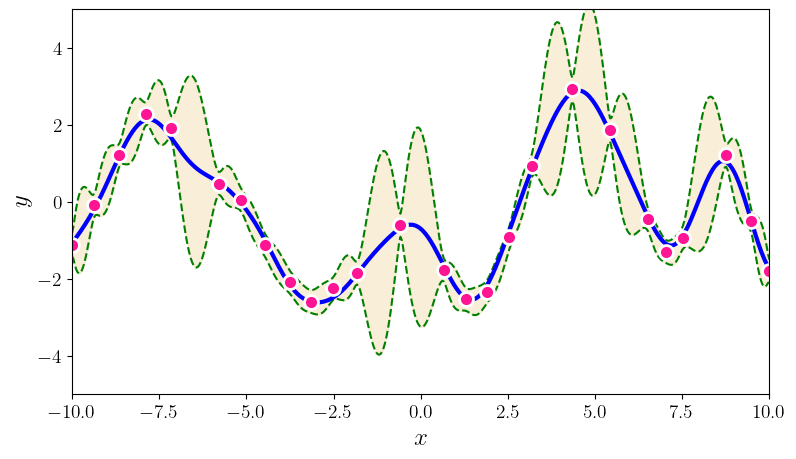}
%   \caption{$n=40$}
\end{subfigure}

\caption{
This figure shows GP-LCB (left column) and CS-LCB (right column) for a correctly specified prior, showing
$t = 3, 7, 18, 25$ (rows 1-4). Here, both methods find the minimizer of $f$,
though GP-LCB has tighter confidence bands and finds the minimizer sooner than CS-LCB.
}
\label{fig:bo_2}

\end{figure}

\newpage

\paragraph{Two dimensional benchmark function}
We also perform a Bayesian optimization experiment on the two dimensional
benchmark \emph{Branin} function.\footnote{Details about this function can
be found here: \url{https://www.sfu.ca/~ssurjano/branin.html}}
In this experiment, we first run Bayesian optimization using the GP-LCB 
algorithm on a model with a misspecified prior, setting $\{\ell=7, \sigma^2=0.1\}$,
and compare it with our CS-LCB algorithm. In both cases, we run each algorithms
for 50 steps, and repeat each algorithm over 10 different seeds. We plot
results of both algorithms in Fig.~\ref{fig:branin}, along with the optimal 
objective value. We find that in this misspecified prior setting, CS-LCB converges
to the minimal objective value more quickly than GP-LCB.

\vspace{5mm}
\begin{figure}[h!]
    \centering
    \includegraphics[width=0.65\textwidth]{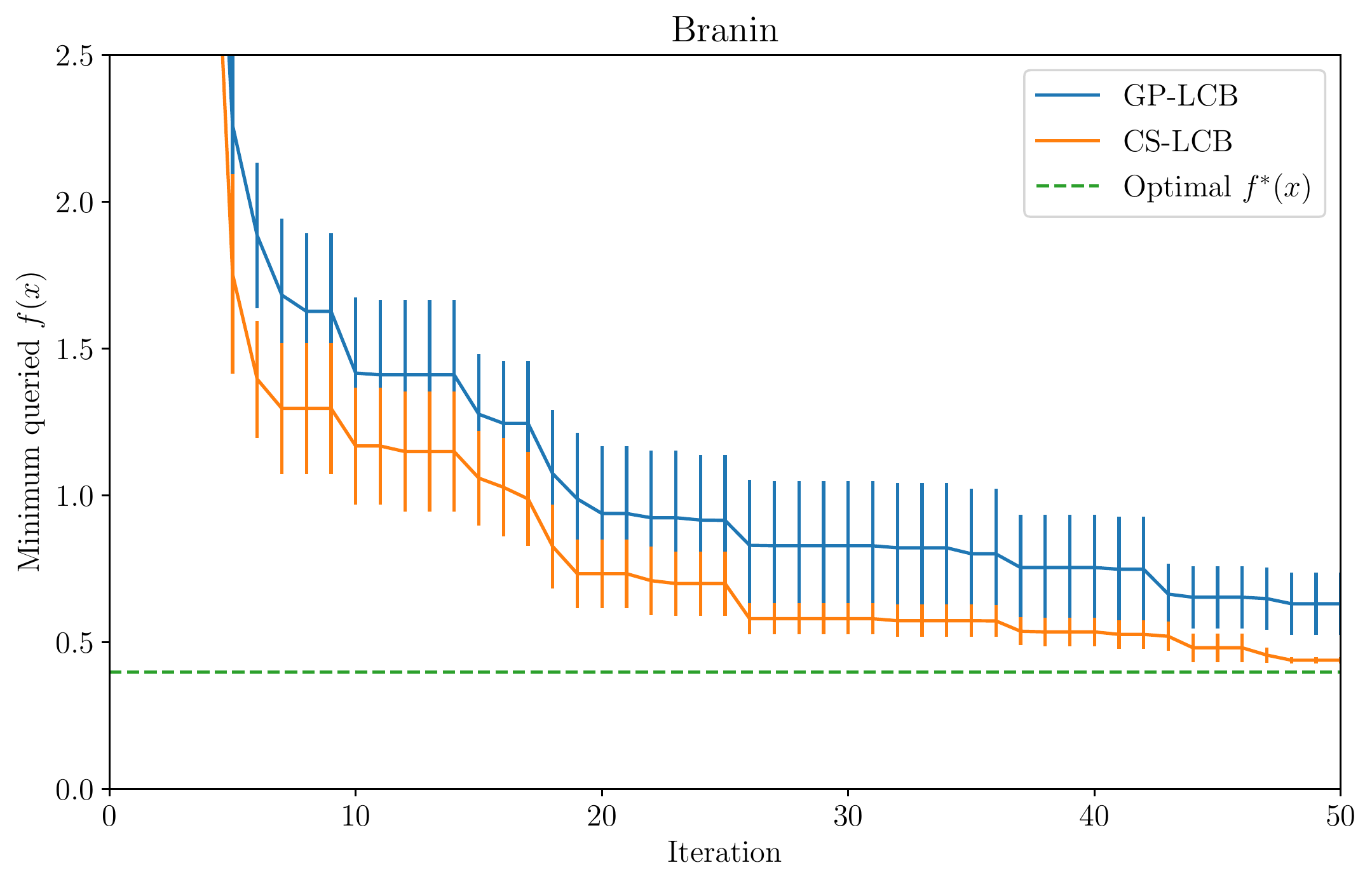}
    \caption{
    Bayesian optimization using CS-LCB and GP-LCB on the Branin function.
    }
    \label{fig:branin}
\end{figure}

\clearpage
\newpage

\section{Misspecified Likelihood: low/high noise, and powered likelihoods}
We next demonstrate BO in the setting where the likelihood is misspecified. In particular,
we are interested in the setting where the model assumes noise $\eta$, which is not
equal to the true noise $\eta^*$ from which the data is generated.
In this case, we demonstrate the fix proposed in 
Section~\ref{sec:simulations}, using powered likelihoods. We show results of this 
adjustment by repeating the experiment of Figure~\ref{fig:viz} for
$\eta > \eta^*$ (Figure~\ref{fig:viz-lownoise}) and $\eta < \eta^*$
(Figures~\ref{fig:viz-highnoise} and \ref{fig:viz-powered}).

\begin{figure}[h!]
\centering

\begin{subfigure}{0.45\textwidth}
  \includegraphics[width=\linewidth]{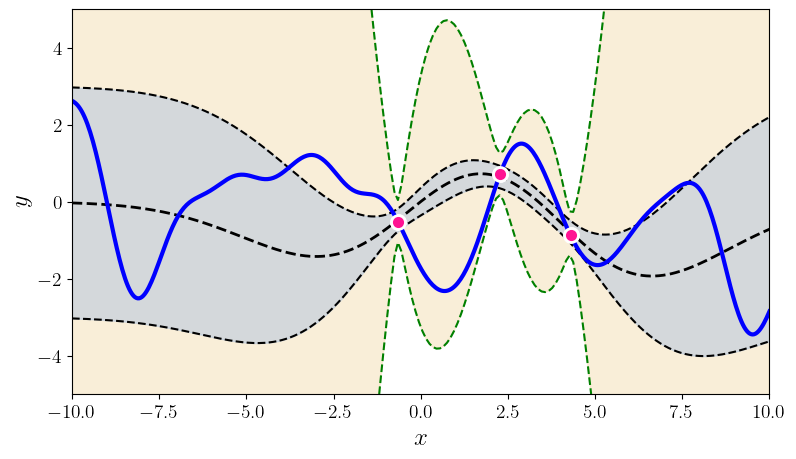}
%   \caption{$n=3$}
\end{subfigure}\hfil
\begin{subfigure}{0.45\textwidth}
  \includegraphics[width=\linewidth]{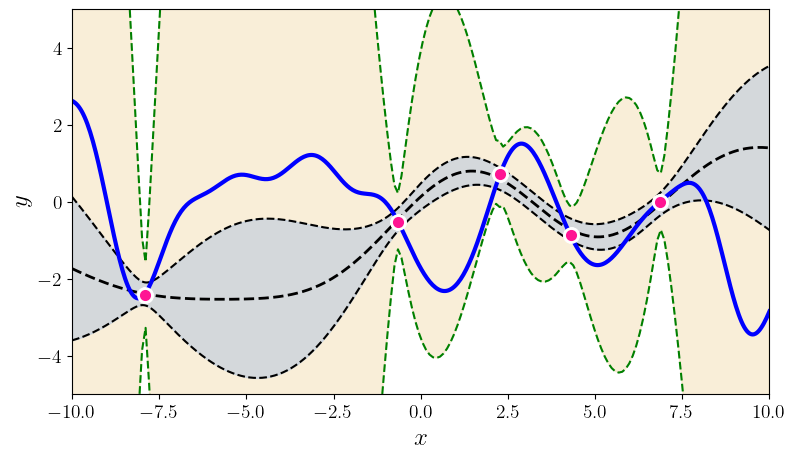}
%   \caption{$n=5$}
\end{subfigure}
\\
\begin{subfigure}{0.45\textwidth}
  \includegraphics[width=\linewidth]{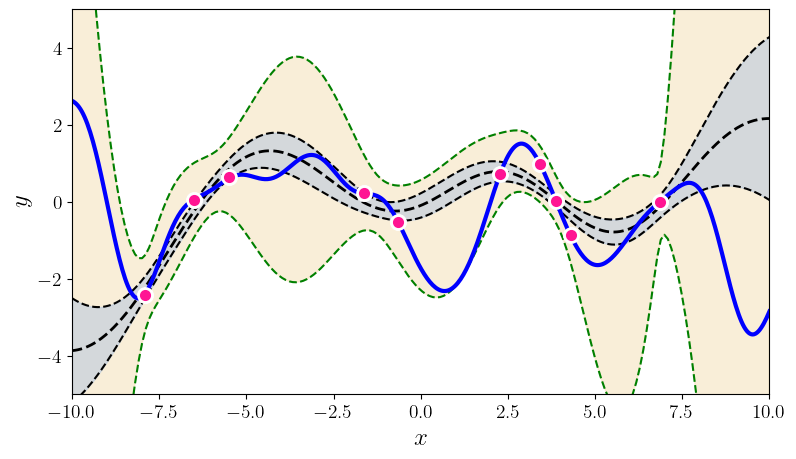}
%   \caption{$n=15$}
\end{subfigure}\hfil
% \medskip
\begin{subfigure}{0.45\textwidth}
  \includegraphics[width=\linewidth]{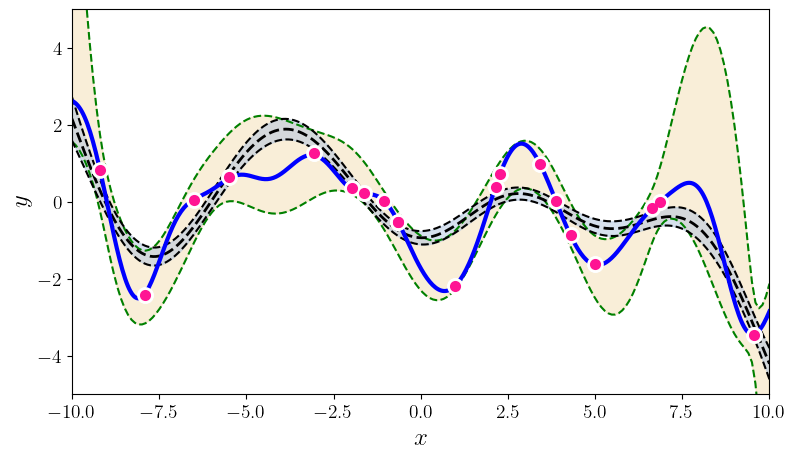}
%   \caption{$n=17$}
\end{subfigure}
\\
\begin{subfigure}{0.45\textwidth}
  \includegraphics[width=\linewidth]{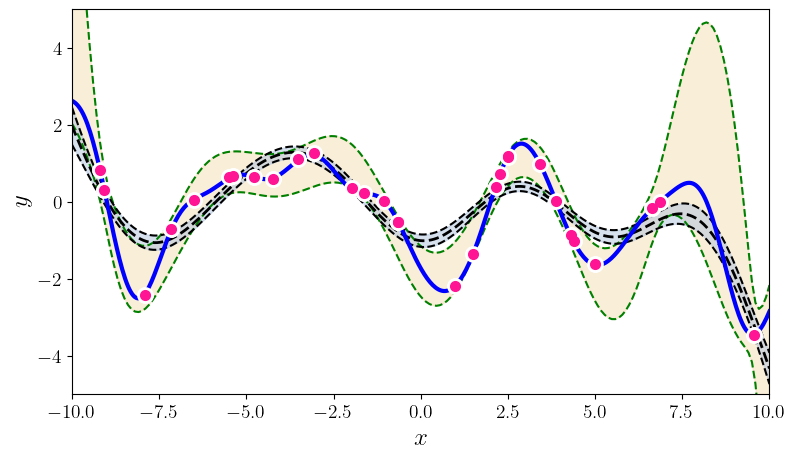}
%   \caption{$n=25$}
\end{subfigure}\hfil
\begin{subfigure}{0.45\textwidth}
  \includegraphics[width=\linewidth]{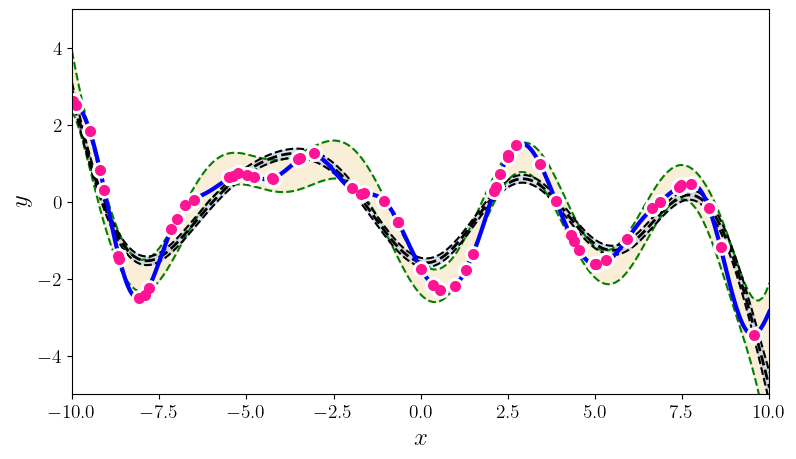}
%   \caption{$n=40$}
\end{subfigure}

\caption{
[Low noise setting] We repeat the experiment of Figure~\ref{fig:viz}, but with the true noise $\eta^*$ of the data being one quarter of the assumed noise $\eta$ in the working model likelihood \eqref{eq:model}. Perhaps as expected, the observed behavior is almost indistinguishable from Figure~\ref{fig:viz} for both the standard GP posterior, which remains incorrectly overconfident, and our method, which covers the true function at all times.
}
\label{fig:viz-lownoise}

\end{figure}

% While the results of the above figure are certainly intuitive, can we justify them mathematically? Calling the true noise $\eta^*$, what we demonstrated in Lemma~\ref{lem:PPR-martingale} was that $\{R_t(f^*, \eta^*)\}$ is a martingale (or in the earlier notation $R_t(f^*)$ is a martingale when $\eta=\eta^*$). One key step in the proof was to just look at the one term
% \[
% \underbrace{\EE_{D_t \sim f^*}\left[ \frac{\phi(\frac{Y_t - g(X_t)}{\eta})}{\phi(\frac{Y_t - f^*(X_t)}{\eta})}\mid \Dcal_{t-1}\right]}_{=1, \text{ when $\eta = \eta^*$}}.
% \]
% We now claim that when $\eta^* \leq \eta$, the above term is upper bounded by one. Indeed,
% \[
% % \EE_{f^*}\left[ \frac{\phi(\frac{Y_t - g(X_t)}{\eta})}{\phi(\frac{Y_t - f^*(X_t)}{\eta})}\mid \Dcal_{t-1}\right] = 
% \int_y \frac{\phi(\frac{y - g(X_t)}{\eta})}{\phi(\frac{y - f^*(X_t)}{\eta})} \frac1{\eta^*}\phi\left(\frac{y - f^*(X_t)}{\eta^*}\right) dy = \int_y \frac1{\eta} \phi\left(\frac{y - g(X_t)}{\eta}\right)  dy = 1,
% \]
% and thus $R_t(f^*,\eta) \leq R_t(f^*, \eta^*)$.

\begin{figure}[h!]
\centering

\begin{subfigure}{0.45\textwidth}
  \includegraphics[width=\linewidth]{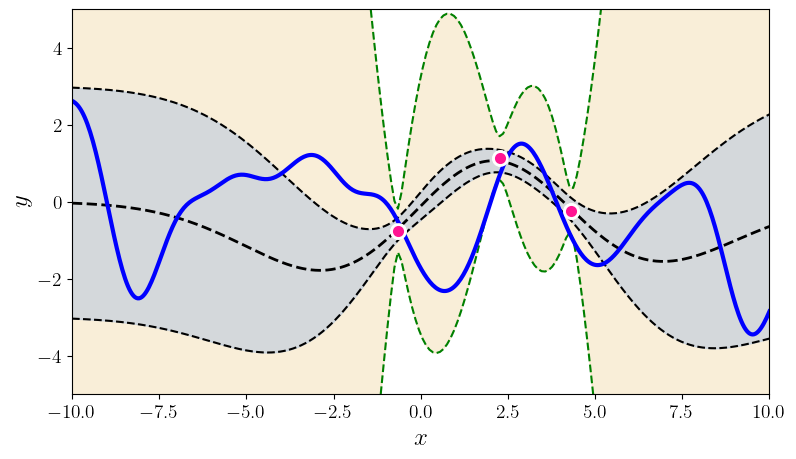}
%   \caption{$n=3$}
\end{subfigure}\hfil
\begin{subfigure}{0.45\textwidth}
  \includegraphics[width=\linewidth]{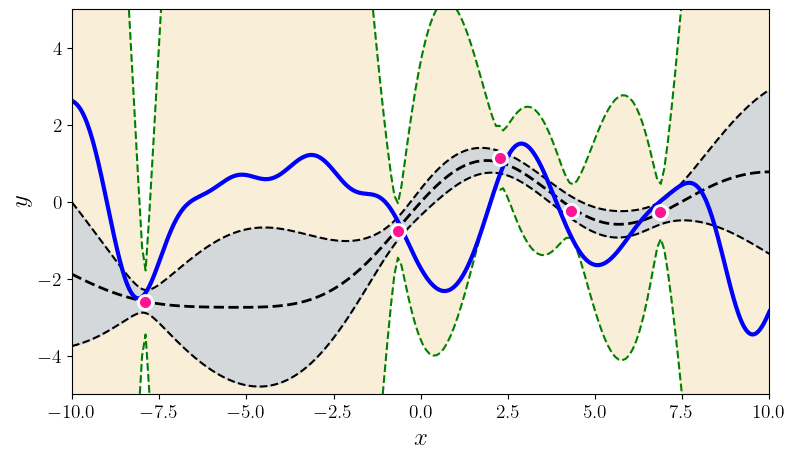}
%   \caption{$n=5$}
\end{subfigure}
\\
\begin{subfigure}{0.45\textwidth}
  \includegraphics[width=\linewidth]{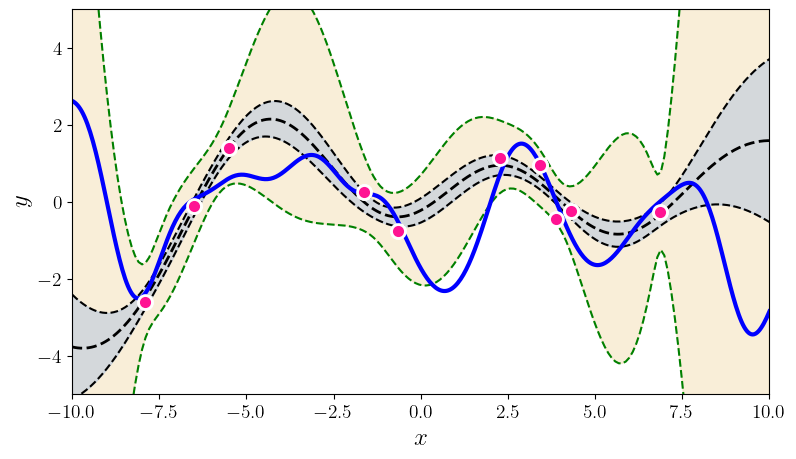}
%   \caption{$n=15$}
\end{subfigure}\hfil
% \medskip
\begin{subfigure}{0.45\textwidth}
  \includegraphics[width=\linewidth]{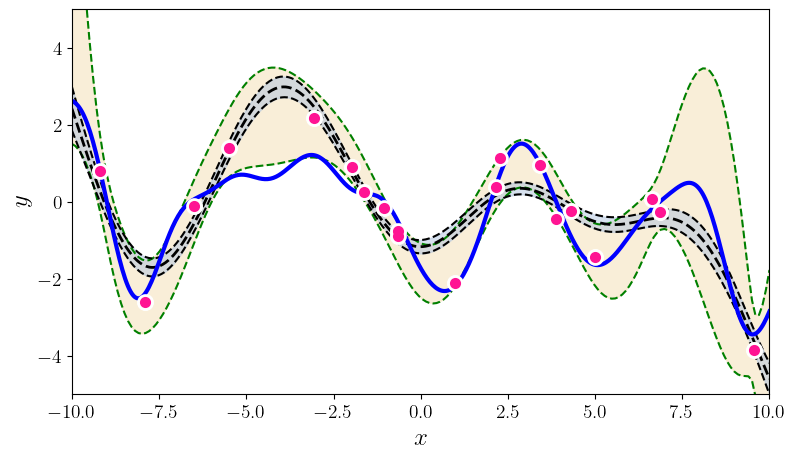}
%   \caption{$n=17$}
\end{subfigure}
\\
\begin{subfigure}{0.45\textwidth}
  \includegraphics[width=\linewidth]{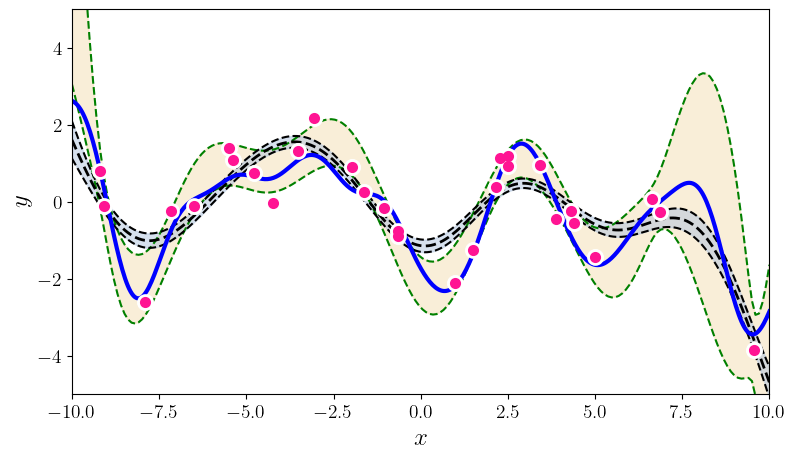}
%   \caption{$n=25$}
\end{subfigure}\hfil
\begin{subfigure}{0.45\textwidth}
  \includegraphics[width=\linewidth]{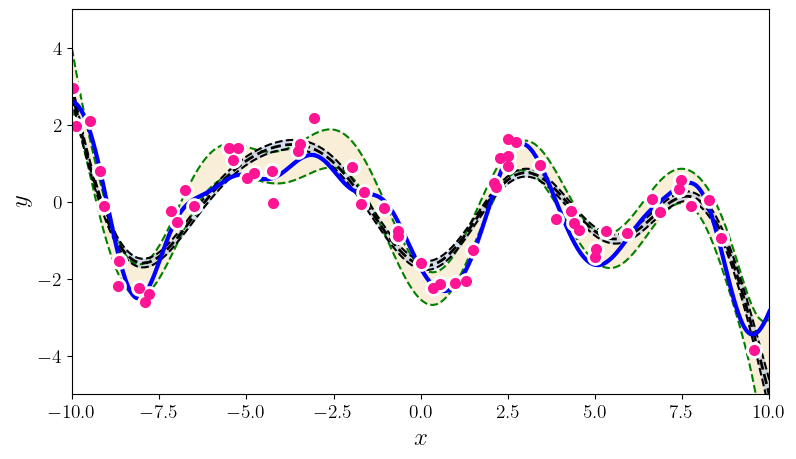}
%   \caption{$n=40$}
\end{subfigure}

% \begin{subfigure}{0.45\textwidth}
%   \includegraphics[width=\linewidth]{01_arxiv/fig/viz_more_noise/viz_3.png}
% %   \caption{$n=3$}
% \end{subfigure}\hfil
% \begin{subfigure}{0.45\textwidth}
%   \includegraphics[width=\linewidth]{01_arxiv/fig/viz_more_noise/viz_5.png}
% %   \caption{$n=5$}
% \end{subfigure}\hfil
% \begin{subfigure}{0.45\textwidth}
%   \includegraphics[width=\linewidth]{01_arxiv/fig/viz_more_noise/viz_15.png}
% %   \caption{$n=15$}
% \end{subfigure}
% % \medskip
% \begin{subfigure}{0.45\textwidth}
%   \includegraphics[width=\linewidth]{01_arxiv/fig/viz_more_noise/viz_17.png}
% %   \caption{$n=17$}
% \end{subfigure}\hfil
% \begin{subfigure}{0.45\textwidth}
%   \includegraphics[width=\linewidth]{01_arxiv/fig/viz_more_noise/viz_25.png}
% %   \caption{$n=25$}
% \end{subfigure}\hfil
% \begin{subfigure}{0.45\textwidth}
%   \includegraphics[width=\linewidth]{01_arxiv/fig/viz_more_noise/viz_40.png}
% %   \caption{$n=40$}
% \end{subfigure}

\caption{
[High noise setting] We repeat the experiment of Figure~\ref{fig:viz}, but with the true noise $\eta^*$ of the data being four times the assumed noise $\eta$ in the working model likelihood \eqref{eq:model}. 
In these plots, we can see incorrect confidence estimates for our prior-robust CS---for example, when the number of
observations $t=10$ (second row, first column), and when $t=20$ (second row, second column). As expected, our prior-robust CS is not robust to misspecification of the likelihood.
}
\label{fig:viz-highnoise}

\end{figure}

\begin{figure}[h!]
\centering

\begin{subfigure}{0.45\textwidth}
  \includegraphics[width=\linewidth]{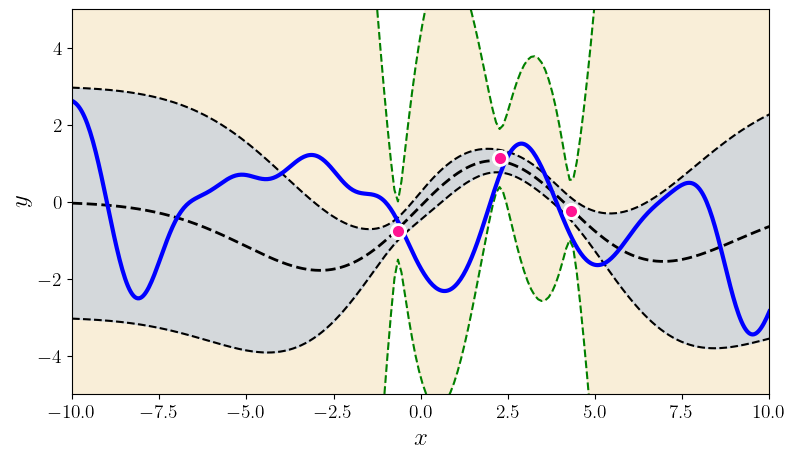}
%   \caption{$n=3$}
\end{subfigure}\hfil
\begin{subfigure}{0.45\textwidth}
  \includegraphics[width=\linewidth]{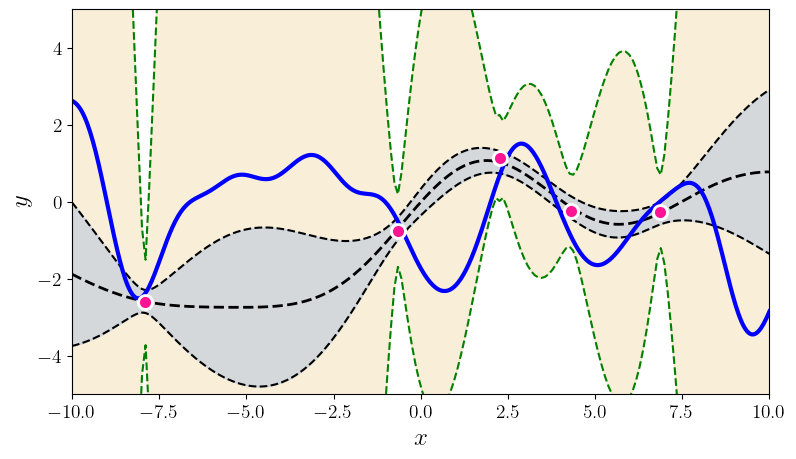}
%   \caption{$n=5$}
\end{subfigure}
\\
\begin{subfigure}{0.45\textwidth}
  \includegraphics[width=\linewidth]{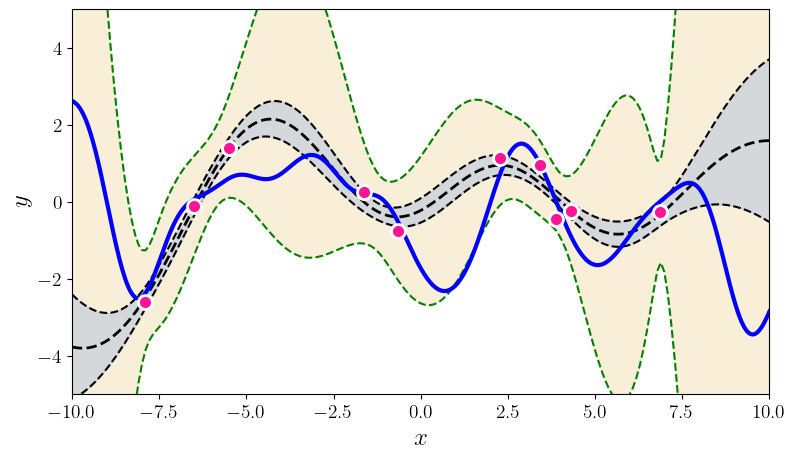}
%   \caption{$n=15$}
\end{subfigure}\hfil
% \medskip
\begin{subfigure}{0.45\textwidth}
  \includegraphics[width=\linewidth]{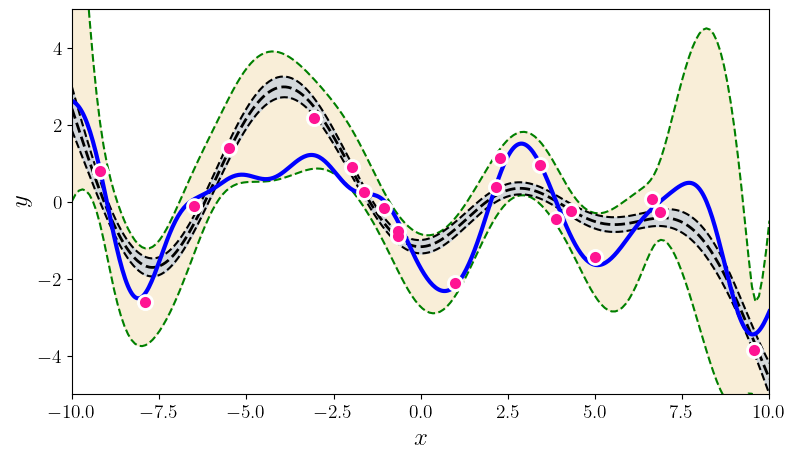}
%   \caption{$n=17$}
\end{subfigure}
\\
\begin{subfigure}{0.45\textwidth}
  \includegraphics[width=\linewidth]{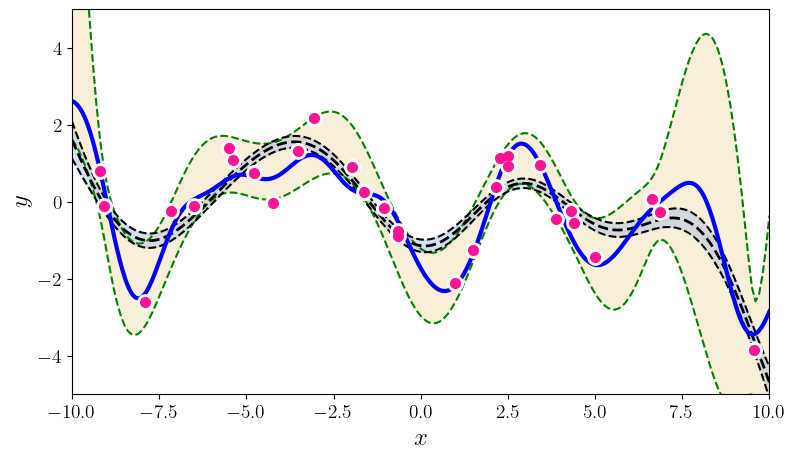}
%   \caption{$n=25$}
\end{subfigure}\hfil
\begin{subfigure}{0.45\textwidth}
  \includegraphics[width=\linewidth]{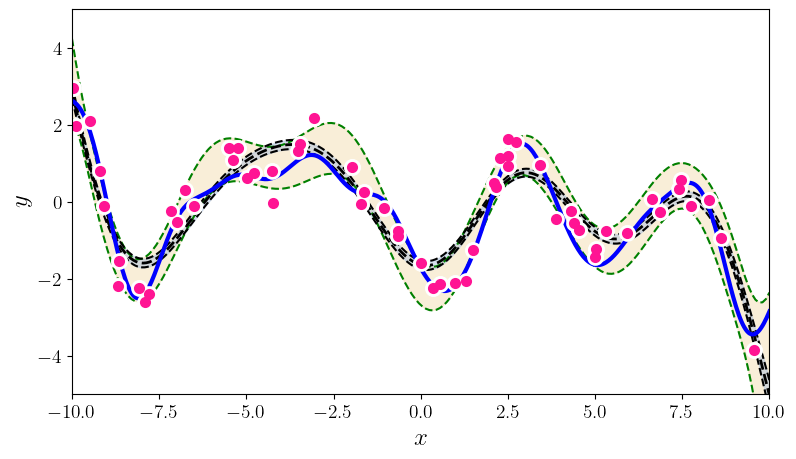}
%   \caption{$n=40$}
\end{subfigure}

\caption{
[High-noise setting with our `powered likelihood' CS]
We consider the same setting of Figure~\ref{fig:viz-highnoise} when the noise of the data is multiplied by four while the assumed noise in the working model likelihood remains the same.
Here, we use a powered likelihood of $\beta = 0.75$ for a more robust confidence sequence, as described at the end of Section~\ref{sec:simulations}.
Note that the earlier issues at $t=10$ (second row, first column) and $t=20$ (second row, second column) 
are now resolved.
}
\label{fig:viz-powered}

% Why does the above idea work?

\end{figure}

\end{document}